\documentclass{article}
\usepackage{natbib}
\usepackage{arxiv}

\usepackage[utf8]{inputenc} 
\usepackage[T1]{fontenc}    
\usepackage{hyperref}       
\usepackage{url}            
\usepackage{booktabs}       
\usepackage{amsfonts}       
\usepackage{nicefrac}       
\usepackage{microtype}      
\usepackage{xcolor}         

\usepackage{url}            
\usepackage{booktabs}       
\usepackage{amscd}
\usepackage{amsmath}
\usepackage{amssymb}
\usepackage{amsthm}
\usepackage{mathrsfs}
\usepackage{mathtools}
\usepackage{algorithm,algorithmicx,algpseudocode}

\newcommand{\alg}[1]{\textsc{#1}}
\usepackage{tikz}
\usepackage{bm}
\usepackage{wrapfig}
\usepackage{subfigure}
\usepackage[export]{adjustbox}
\usepackage{multirow}
\newcommand{\set}[1]{\{ #1 \}}
\newcommand{\ind}[1]{\mathbb{I}\left[ #1 \right]}

\renewcommand{\epsilon}{\varepsilon}
\newcommand{\bx}{\bm{x}}
\newcommand{\hx}{\hat{\bm{x}}}
\newcommand{\tx}{\tilde{\bm{x}}}
\newcommand{\ba}{\bm{a}}
\newcommand{\calA}{\mathcal{A}}

\newtheorem{proposition}{Proposition}
\newtheorem{lemma}{Lemma}
\theoremstyle{definition}
\newtheorem{problem}{Problem}

\newtheorem{remark}{Remark}

\usepackage{cleveref}
\crefname{algorithm}{Algorithm}{Algorithms}
\crefname{table}{Table}{Tables}
\crefname{figure}{Figure}{Figures}
\crefname{section}{Section}{Sections}
\crefname{appendix}{Appendix}{Appendices}
\crefname{problem}{Problem}{Problems}
\crefname{theorem}{Theorem}{Theorems}
\crefname{proposition}{Proposition}{Propositions}
\crefname{lemma}{Lemma}{Lemmas}
\crefname{remark}{Remark}{Remarks}
\crefname{algorithm}{Algorithm}{Algorithms}
\newcommand{\myparagraph}[1]{ \noindent {\mbox{\textbf{#1}~~}}}

\title{Algorithmic Recourse with Missing Values}

\author{%
  Kentaro Kanamori \\
  Fujitsu Ltd. \\
  \texttt{k.kanamori@fujitsu.com} \\
  \And
  Takuya Takagi \\
  Fujitsu Ltd. \\
  \texttt{takagi.takuya@fujitsu.com} \\
  \AND
  Ken Kobayashi \\
  Tokyo Institute of Technology \\
  \texttt{kobayashi.k.ar@m.titech.ac.jp} \\
  \And
  Yuichi Ike \\
  Kyushu University \\
  \texttt{ike@imi.kyushu-u.ac.jp} \\
}

\begin{document}

\maketitle

\begin{abstract}
This paper proposes a new framework of algorithmic recourse (AR) that works even in the presence of missing values. 
AR aims to provide a recourse action for altering the undesired prediction result given by a classifier. 
Existing AR methods assume that we can access complete information on the features of an input instance. 
However, we often encounter missing values in a given instance (e.g., due to privacy concerns), and previous studies have not discussed such a practical situation. 
In this paper, we first empirically and theoretically show the risk that a naive approach with a single imputation technique fails to obtain good actions regarding their validity, cost, and features to be changed. 
To alleviate this risk, we formulate the task of obtaining a valid and low-cost action for a given incomplete instance by incorporating the idea of multiple imputation. 
Then, we provide some theoretical analyses of our task and propose a practical solution based on mixed-integer linear optimization. 
Experimental results demonstrated the efficacy of our method in the presence of missing values compared to the baselines. 
\end{abstract}

\section{Introduction}\label{sec:intro}
Algorithmic decision-making with machine learning models has been applied to various tasks in the real world, such as loan approvals. 
In such critical tasks, the predictions made by a model might have a significant impact on individual users~\citep{Rudin:NMI2019}. 
Consequently, decision-makers need to explain how individuals should act to alter the undesired decisions~\citep{Miller:AI2019,Wachter:HJLT2018}. 
\emph{Algorithmic Recourse~(AR)} aims to provide such information~\citep{Ustun:FAT*2019}. 
For a classifier $h \colon \mathcal{X} \to \mathcal{Y}$, a desired class $y^\ast \in \mathcal{Y}$, and an instance $\bx \in \mathcal{X}$, AR provides a perturbation $\ba$ that flips the prediction result into the desired class, i.e., $h(\bx + \ba) = y^\ast$, with a minimum effort measured by some cost function $c$. 
The user can regard the perturbation $\ba$ as a \emph{recourse action} for obtaining the desired outcome $y^\ast$ from $h$~\citep{Karimi:ACMCS2022}. 

In this paper, we consider a situation where an instance $\bx$ includes \emph{missing values}. 
In practical situations, missing values arise not only when users do not know their feature values~\citep{Cesa-Bianchi:JMLR2012} but also when users do not input their values on purpose, e.g., due to their privacy concerns~\citep{Schenker:JASA2006}. 
Indeed, some recent studies have pointed out the risk that an adversary can leverage recourse actions to infer private information about the data held by decision-makers~\citep{Pawelczyk:AISTATS2023}. 
To alleviate this risk, it is important to allow users to avoid disclosing some feature values corresponding to their private information (e.g., their income). 
However, almost all of the existing AR methods assume that the complete information of an input instance $\bx$ is given~\citep{Verma:arxiv2020,Guidotti:DMKD2022}. 
Therefore, we need to discuss how to handle missing values in AR and develop a new framework that works without acquiring the true value of the missing features. 

A common way to handle missing values is \emph{imputation} that replaces them with plausible values~\citep{Little:2019:Missing}. 
In the context of AR, however, imputation can often affect resulting actions, as demonstrated below. 
\cref{fig:intro:demo} presents an example of a synthetic loan approval task. 
For an instance $\bx$ in \cref{fig:intro:demo:original}, we drop the value of the feature ``Income" (\$30K) and impute it with the empirical mean (\$66K), which we write for $\hx$ in \cref{fig:intro:demo:imputed}. 
Using the existing AR method for the imputed instance $\hx$, we obtain an optimal action $\ba$ that decreases the feature ``\#ExistingLoans" by $1$. 
As shown in \cref{fig:intro:demo:boundary}, while the action $\ba$ is valid for $\hx$, i.e., $h(\hx + \ba) = y^\ast$, it is not valid for $\bx$, i.e., $h(\bx + \ba) \not= y^\ast$. 
It suggests that valid actions for an imputed instance $\hx$ can be different from those for its original instance $\bx$, even if the prediction results before and after imputation are the same, i.e., $h(\bx) = h(\hx)$. 

Based on the above observations, we argue that decision-makers should provide users with recourse actions that are valid not for the imputed instances $\hx$ but for their original instances $\bx$. 
Even if a suggested action is valid for the imputed instance $\hx$, it may lead to risky decision-making because it is not commensurate with the user's true state $\bx$~\citep{Harris:NIPS2022,Olckers:arxiv2023}. 
For the example in \cref{fig:intro:demo}, if a user implements the suggested action $\ba$, the user can get the loan approved by $h$ because it is valid for $\hx$. 
However, since the feature ``Income" of $\hx$ is imputed with a larger value than the true value of $\bx$ and the action $\ba$ is not valid for $\bx$ as shown in \cref{fig:intro:demo:boundary}, the decision-maker approves the loan beyond the user's actual capacity to repay, which may result in default. 
To alleviate this risk, we aim to obtain a valid action for the original instance $\bx$ using only a given instance with missing values.

\begin{figure}[t]
    \centering
    \subfigure[Original instance $\bx$]{
        \adjustbox{valign=b}{
            \begin{tabular}{lc}
            \toprule
                \textbf{Features} & \textbf{Values} \\
            \midrule
               Age & 32 \\
               Income & \$30K \\
               Purpose & NewCar \\
               Education & Masters \\
               \#ExistingLoans & 2 \\
            \bottomrule
            \vspace{1mm}
            \end{tabular}
            \label{fig:intro:demo:original}    
        }
    }
    \hfill
    \subfigure[Imputed instance $\hx$]{
        \adjustbox{valign=b}{
            \begin{tabular}{lc}
            \toprule
                \textbf{Features} & \textbf{Values} \\
            \midrule
               Age & 32 \\
               Income & \underline{\$66K} \\
               Purpose & NewCar \\
               Education & Masters \\
               \#ExistingLoans & 2 \\
            \bottomrule
            \vspace{1mm}
            \end{tabular}
            \label{fig:intro:demo:imputed}    
        }
    }
    \hfill
    \subfigure[Decision boundary of classifier $h$]{
        \includegraphics[width=0.35\textwidth,valign=b]{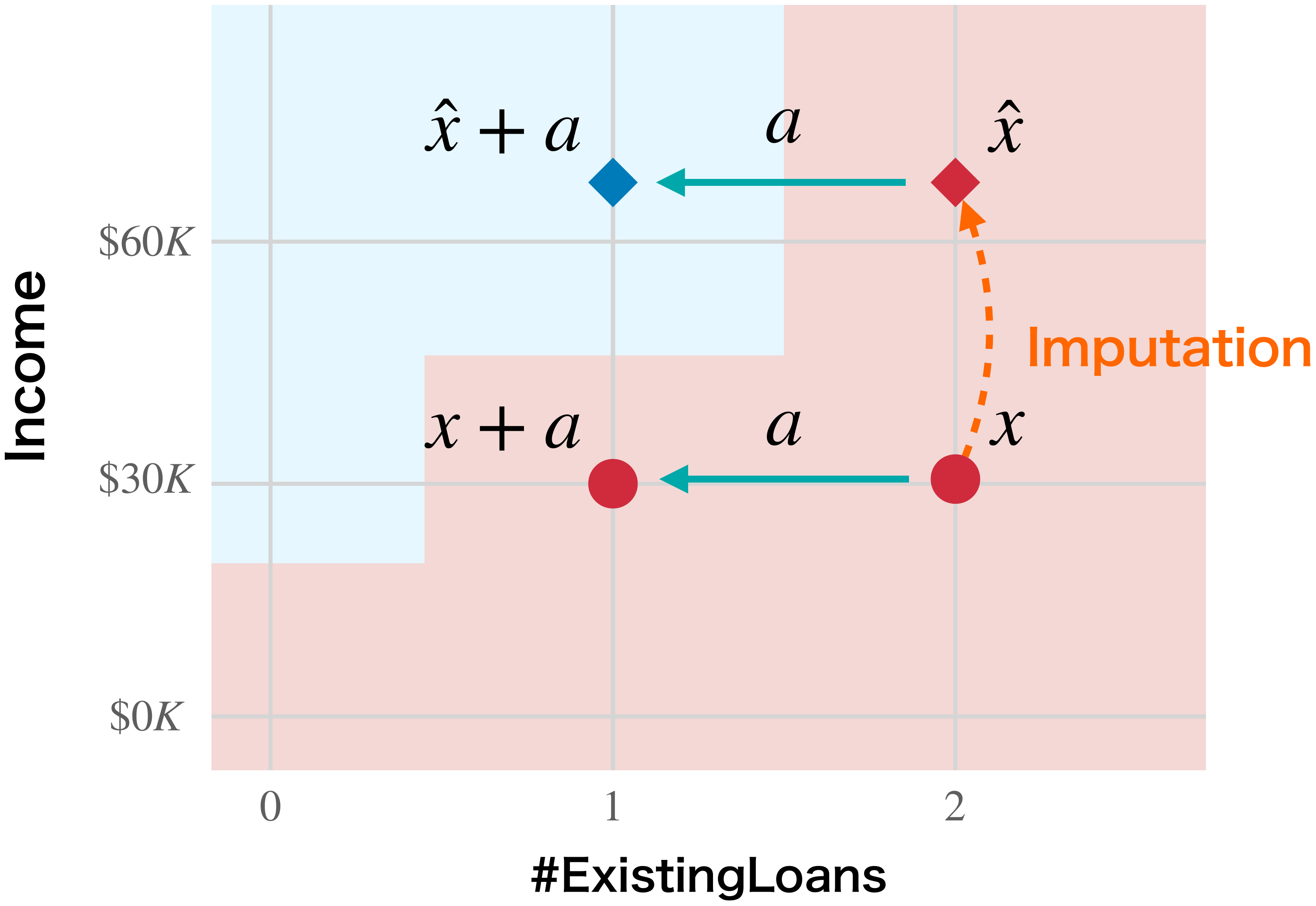}
        \label{fig:intro:demo:boundary}    
    }
    \caption{
        Examples of an original instance $\bx$, its imputed instance $\hx$, and the decision boundary of a classifier $h$. 
        Here, we drop the feature ``Income" in $\bx$ as a missing feature and obtain $\hx$ by imputing its value with the empirical mean \$66K. 
        For the imputed instance $\hx$, we obtain an optimal action $\ba$ using the existing AR method. 
        While the action $\ba$ successfully alters the prediction result of $h$ for the imputed instance $\hx$, it fails to do that for the original instance $\bx$. 
    }
    \label{fig:intro:demo}
\end{figure}

\myparagraph{Our contributions.}
This paper proposes the first framework of AR that works even in the presence of missing values, named \emph{Algorithmic Recourse with Multiple ImputationN (ARMIN)}. 
By incorporating the idea of \emph{multiple imputation}, which considers multiple ways of plausible imputation, into AR, our ARMIN can provide valid actions for a given incomplete instance without acquiring its missing values with high probability. 
Our contributions are summarized as follows:

\begin{enumerate}
    \item 
    We are the first to tackle the problem of AR with missing values. 
    We empirically and theoretically show the risk that the resulting action is highly affected by imputation methods. 
    \item 
    We formulate our task of obtaining a valid and low-cost action for a given incomplete instance by leveraging the idea of multiple imputation. 
    Then, we propose a practical solution to our task based on mixed-integer linear optimization. 
    \item 
    By experiments on real datasets, we demonstrated that our ARMIN provides incomplete instances with quantitatively and qualitatively better actions than existing AR methods. 
\end{enumerate}

\subsection{Related work}
\emph{Algorithmic Recourse (AR)}, also referred to as \emph{Counterfactual Explanation}, has attracted increasing attention~\citep{Verma:arxiv2020}, and several methods for providing recourse actions have been proposed~\citep{Wachter:HJLT2018,Ustun:FAT*2019,Mothilal:FAT*2020,Kanamori:AAAI2021,Pawelczyk:WWW2020,Karimi:AISTATS2020}. 
On the other hand, recent studies pointed out various issues, e.g., causality~\citep{Karimi:NIPS2020,Karimi:FAccT2021}, privacy~\citep{Pawelczyk:AISTATS2023}, transparency~\citep{Rawal:NIPS2020,Kanamori:AISTATS2022}, and so on~\citep{Barocas:FAT*2020,Venkatasubramanian:FAT*2020}. 
One of the critical problems is robustness to the perturbations of features~\citep{Slack:NIPS2021,Dominguez-olmedo:ICML2022}. 
These studies are motivated by the fact that actions are often affected by small changes to inputs, which is similar to our observation that actions are affected by imputation. 
Although the existing robust AR methods can be applied to the setting with missing values if we replace its uncertainty set with imputation candidates, the resulting action often has a high cost~\citep{Pawelczyk:arxiv2022}. 
As shown later, our experiments demonstrate that our ARMIN can provide valid actions for incomplete instances with lower costs than the robust AR methods. 

\emph{Missing data analysis} is a traditional branch of statistics because real datasets often contain missing values~\citep{Rubin:Biometrika1976,Little:2019:Missing}. 
In the literature on machine learning, there are several studies not only on imputation with deep generative models~\citep{Yoon:ICML2018,Mattei:ICML2019}, but also on the impacts of imputation on prediction consistency~\citep{Josse:arxiv2019,LeMorvan:NIPS2021,Ayme:ICML2022} and predictive fairness~\citep{Zhang:NIPS2021,Jeong:AAAI2022}. 
However, there is little work that studies their impacts on explanation, including recourse actions, while its importance has been recognized~\citep{Ahmad:IJCAIWS2019,Verma:arxiv2020,Guidotti:DMKD2022}. 
To the best of our knowledge, our work is the first to point out the issues with missing values for AR and propose a concrete method for addressing these issues.

\section{Preliminaries}\label{sec:prelim}
For a positive integer $n \in \mathbb{N}$, we write $[n] \coloneqq \set{1,\dots,n}$.
As with the previous studies~\citep{Ustun:FAT*2019}, we consider a binary classification setting between undesired and desired classes.  
We denote input and output domains $\mathcal{X} = \mathcal{X}_1 \times \dots \times \mathcal{X}_D \subseteq \mathbb{R}^{D}$ and $\mathcal{Y} = \set{\pm 1}$, respectively. 
We call a vector $\bx = (x_1, \dots, x_D) \in \mathcal{X}$ an \emph{instance}, and a function $h \colon \mathcal{X} \to \mathcal{Y}$ a \emph{classifier}. 
Without loss of generality, we assume that $h(\bx)=+1$ is a desirable outcome for users (e.g., loan approval). 

\subsection{Algorithmic recourse}
For an instance $\bx \in \mathcal{X}$, we define an \emph{action} as a perturbation vector $\ba \in \mathbb{R}^{D}$ such that $\bx+\ba \in \mathcal{X}$. 
Let $\calA$ be a set of pre-defined feasible actions for $\bx$ such that $\bm{0} \in \calA$ and $\calA \subseteq \{ \ba \in \mathbb{R}^{D} \mid \bx + \ba \in \mathcal{X} \}$. 
For a classifier $h$, an action $\ba$ is \emph{valid} for $\bx$ if $\ba \in \calA$ and $h(\bx + \ba)=+1$. 
For $\ba \in \calA$, a \emph{cost function} $c \colon \calA \to \mathbb{R}_{\geq 0}$ measures the required effort for implementing the action $\ba$. 

For a given classifier $h \colon \mathcal{X} \to \mathcal{Y}$ and an instance $\bx \in \mathcal{X}$, the aim of \emph{Algorithmic Recourse (AR)}~\citep{Karimi:ACMCS2022} is to find an action $\ba$ that is valid for $\bx$ with respect to $h$ and minimizes its cost $c(\ba)$; that is, 
\begin{align}\label{eq:ce}
    {\mathop{\text{\upshape minimize}}}_{\ba \in \calA} \; c(\ba) \;\;\; \text{\upshape subject to} \;\; h(\bx + \ba)=+1.
\end{align}
Hereafter, we fix $h$ and $c$ and omit them if they are clear from the context. 
Note that our proposed formulation in \cref{sec:formulation} does not depend on $h$ and $c$, and in \cref{sec:optimization}, we propose a concrete solution for specific classifiers $h$ and cost functions $c$ that are commonly used in the previous studies on AR.

\subsection{Missing values}
In practice, an input instance may contain features with \emph{missing values}~\citep{Little:2019:Missing}. 
Let $*$ be a symbol indicating a missing value. 
For an original complete instance $\bx \in \mathcal{X}$, we denote its incomplete instance by $\tx = (\tx_1, \dots, \tx_D)$ with some missing values $*$, where
\begin{align*}
    \tx_d = 
    \begin{cases}
        *, & \text{if the feature } d \text{ is missing}, \\
        x_d, & \text{otherwise},
    \end{cases}
\end{align*}
for $d \in [D]$. 
We denote the input domain with missing values $\tilde{\mathcal{X}} \coloneqq (\mathcal{X}_1 \cup \set{*}) \times \dots \times (\mathcal{X}_D \cup \set{*})$. 

Let $M_{\tx} = \set{d \in [D] \mid \tx_d = *}$ (resp.\ $O_{\tx} = [D] \setminus M_{\tx}$) be the set of features that are missing (resp.\ observed). 
In statistics, mechanisms of missing values are categorized into three types depending on the relationship between $M_{\tx}$ and $O_{\tx}$~\citep{Rubin:Biometrika1976}: 
(i)~\emph{missing completely at random (MCAR)} if $M_{\tx}$ is independent of $[D]$; 
(ii)~\emph{missing at random (MAR)} if $M_{\tx}$ depends only on $O_{\tx}$; 
(iii)~\emph{missing not at random (MNAR)} if neither MCAR nor MAR holds. 
Note that existing methods for missing data analysis often rely on the MAR or MCAR assumption for their soundness~\citep{Little:2019:Missing}. 
In our experiments of \cref{sec:experiments}, we evaluated the efficacy of our method in each situation. 

A common approach to handling missing values is \emph{imputation} that replaces them with plausible values and obtains imputed instances $\hx \in \mathcal{X}$. 
There exist several practical methods, such as \emph{multivariate imputation by chained equations~(MICE)}~\citep{Buuren:JSS2011} and \emph{$k$-nearest neighbor~($k$-NN) imputation}~\citep{Troyanskaya:Bioinformatics2001}.

\section{Problem formulation}\label{sec:formulation}
In this section, we formulate our task of providing a recourse action to a given instance with missing values. 
Let us consider the situation where we have an incomplete instance $\tx \in \tilde{\mathcal{X}}$ that comes from the original complete instance $\bx \in \mathcal{X}$. 
We assume that a user intentionally hides the true values of some features and that we cannot access them. 
Our aim is to obtain a valid and low-cost action $\ba$ for $\bx$ using only the given instance $\tx$ with missing values. 
In this section, we first consider a naive approach using a single imputation technique and analyze its risk by showing its theoretical properties. 
Motivated by these results, we introduce our approach by incorporating the idea of multiple imputation into AR. 
All the proofs of the statements are presented in \cref{sec:appendix:proof}. 

\subsection{Naive formulation with single imputation and its drawback}
A naive approach is to obtain an imputed instance $\hx$ by applying an imputation method to $\tx$ and then optimize an action by solving the problem \eqref{eq:ce} for $\hx$. 
This can be formulated as follows:
\begin{align}\label{eq:naive}
    {\mathop{\text{\upshape minimize}}}_{\ba \in \calA} \; c(\ba) \;\;\; \text{\upshape subject to} \;\; h(\hx + \ba)=+1. 
\end{align}
We denote optimal actions to \eqref{eq:ce} and \eqref{eq:naive} by $\ba^\ast$ and $\hat{\ba}$, respectively, and show some theoretical relationships between them. 
First, we show two trivial properties on the validity and cost of $\hat{\ba}$: 

\begin{remark}\label{rem:invalid}
    If $h(\bx) = h(\hx) = -1$ and $c(\hat{\ba}) < c(\ba^\ast)$, then $\hat{\ba}$ is not valid for $\bx$. 
\end{remark}
\begin{remark}\label{rem:nullaction}
    If $h(\bx) = -1$ and $h(\hx) = +1$, then $\hat{\ba} = \bm{0}$ and $\hat{\ba}$ is not valid for $\bx$ since $h(\bx + \hat{\ba}) \not= +1$. 
    Furthermore, if $h(\bx) = +1$ and $h(\hx) = -1$, then $c(\hat{\ba}) \geq c(\ba^\ast)$. 
\end{remark}

\cref{rem:invalid} implies that an optimal action $\hat{\ba}$ for the imputed instance $\hx$ is not valid for its original instance $\bx$ if its cost is less than that of $\ba^\ast$, as demonstrated in \cref{tab:exp:marexample} later. 
\cref{rem:nullaction} implies that if the prediction is changed by imputation, i.e., $h(\bx) \not= h(\hx)$, then $\hat{\ba}$ is not valid for $\bx$ or has a higher cost than $\ba^\ast$. 
These results indicate the risk that imputation affects the validity and cost of actions. 

Next, we analyze how much an optimal action $\hat{\ba}$ after imputation differs from $\ba^\ast$. 
We consider the same setting as \citep{Ustun:FAT*2019}.
Let $h_{\bm{\beta}}$ be a linear classifier defined as $h_{\bm{\beta}}(\bx) = \operatorname{sgn}(f_{\bm{\beta}}(\bx))$ with the decision function $f_{\bm{\beta}}(\bx) = \bm{\beta}^\top \bx$ and coefficient vector $\bm{\beta} = (\beta_1, \dots \beta_D)  \in \mathbb{R}^D$. 
We assume $\mathcal{X}=\mathbb{R}^D$, $\calA = \mathbb{R}^D$, and $c(\ba) = \| \ba \|$. 
We also assume that a single feature $d^\circ \in [D]$ is missing and that the missing value of $d^\circ$ is imputed with the population mean $\mu_{d^\circ} = \mathbb{E}_{\bx} \left[ x_{d^\circ} \right]$ over a distribution on the input domain $\mathcal{X}$. 
Note that these assumptions are common in analyses with missing values~\citep{Bertsimas:arxiv2021,Josse:arxiv2019}. 
In \cref{prop:upper}, we give an upper bound on the expected difference between $\ba^\ast$ and $\hat{\ba}$. 

\begin{proposition}\label{prop:upper}
    For an instance $\bx \in \mathcal{X}$ and a feature $d^\circ \in [D]$, let $\hx \in \mathcal{X}$ be its imputed instance with $\hat{x}_{d^\circ} = \mu_{d^\circ}$ and $\hat{x}_d = x_d$ for $d \in [D] \setminus \set{d^\circ}$. 
    For $\bx$ and $\hx$, let $\ba^\ast$ and $\hat{\ba}$ be optimal solutions to the problems \eqref{eq:ce} and \eqref{eq:naive} with respect to a linear classifier $h_{\bm{\beta}}$, respectively. 
    Then, we have
    \begin{align*}
        \mathbb{E} \left[ \| \hat{\ba} - \ba^\ast \|_2^2 \right] 
        \leq 
        \frac{1}{\| \bm{\beta} \|_2^2}
        \left(
        \beta_{d^\circ}^2 \cdot \sigma_{d^\circ}^{2}
        + 
        \gamma \cdot p_\mathrm{conf}
        \right), 
    \end{align*}
    where $\sigma_{d^\circ}^2 = \mathbb{E}[ (x_{d^\circ} - \mu_{d^\circ})^2 ]$, 
    $\gamma = \mathbb{E}[ \left( f_{\bm{\beta}}(\bx) \right)^2 \mid h_{\bm{\beta}}(\bx) \not= h_{\bm{\beta}}(\hx)]$,
    and $p_\mathrm{conf} = \mathbb{P}[h_{\bm{\beta}}(\bx) \not= h_{\bm{\beta}}(\hx)]$. 
\end{proposition}

\cref{prop:upper} implies that an upper bound on the expected difference between $\ba^\ast$ and $\hat{\ba}$ depends on the variance $\sigma_{d^\circ}^2$ of a missing feature $d^\circ$ and the probability $p_\mathrm{conf}$ that the prediction $h_{\bm{\beta}}(\bx)$ is changed by imputation. 
In practice, as demonstrated in \cref{tab:exp:marexample} later, features included in the resulting action may be changed by imputation. 
These results suggest the risk that the resulting action $\hat{\ba}$ would be far from $\ba^\ast$, even if we impute missing values with the population mean. 

In summary, we empirically and theoretically find that imputation can affect the resulting action in terms of its validity, cost, and features. 
Note that some models, such as tree-based models like XGBoost~\citep{Chen:KDD2016}, can handle missing values without imputation. 
Even for such models, however, we show that actions are highly affected depending on the presence of missing values in \cref{sec:appendix:experiments:gbdt}. 

\subsection{Our formulation with multiple imputation}
We introduce our approach, named \emph{algorithmic recourse with multiple imputation (ARMIN)}. 
One of the reasons why our task is difficult to solve is that we cannot evaluate whether an action $\ba$ is valid for the original instance $\bx$, i.e., $h(\bx + \ba) = +1$, since we cannot access to $\bx$. 
While a single imputation technique is known to be able to consistently estimate the prediction result $h(\bx)$ for a given incomplete instance $\tx$~\citep{Josse:arxiv2019,LeMorvan:NIPS2021,Ayme:ICML2022}, our results empirically demonstrate that it often fails to estimate the prediction result $h(\bx + \ba)$ after implementing an action $\ba$. 
To overcome this difficulty, we propose to estimate the validity of $\ba$ for $\bx$ by leveraging the technique of \emph{multiple imputation}~\citep{Rubin:2004MI}, which considers multiple ways of plausible imputation to estimate the validity of $\ba$ for $\bx$. 

For a given instance $\tx$ with missing values, we define the space of plausible imputation candidates for $\tx$ as $\mathcal{I}_{\tx} \coloneqq  \mathcal{I}_1 \times \dots \times \mathcal{I}_D \subseteq \mathcal{X}$, where 
\begin{align*}
    \mathcal{I}_d \coloneqq 
    \begin{cases}
        \mathcal{X}_d & \text{if } \tx_d = *, \\
        \set{\tx_d} & \text{otherwise},
    \end{cases}
\end{align*}
for $d \in [D]$. 
We call $\mathcal{I}_{\tx}$ the \emph{imputation space} of $\tx$. 
Then, our formulation of the task is as follows. 

\begin{problem}\label{prob:armin}
    Given a classifier $h \colon \mathcal{X} \to \set{ \pm 1 }$, an instance $\tx \in \tilde{\mathcal{X}}$ with missing values, and a confidence parameter $\rho \in [0, 1]$, find an optimal solution $\ba^\ast \in \calA$ to the following problem:
    \begin{align*}
        {\mathop{\text{\upshape minimize}}}_{\ba \in \calA} \; c(\ba) \;\;\; \text{\upshape subject to} \;\; V_{\tx}(\ba) \geq \rho, 
    \end{align*}
    where $V_{\tx}(\ba) \coloneqq \mathbb{E}_{\hx \sim \mathcal{D}_{\tx}} [\ind{h(\hx + \ba) = +1}]$ is the \emph{expected validity} of an action $\ba$ for $\tx$ over a distribution $\mathcal{D}_{\tx}$ of its imputation space $\mathcal{I}_{\tx}$. 
\end{problem}

In our formulation, we estimate the validity of an action $\ba$ for the original instance $\bx$ using the expected validity $V_{\tx}(\ba)$, which corresponds to the expected probability that $\ba$ is valid over the imputation space $\mathcal{I}_{\tx}$. 
This idea is inspired by ``Rubin's rules" in the framework of multiple imputation~\citep{Rubin:2004MI} that estimates some statistic of the true dataset by generating multiple imputed datasets and averaging their analysis results~\citep{Dick:ICML2008,Ipsen:ICLR2022}. 
We can expect that the larger the expected validity $V_{\tx}(\ba)$ is, the higher the probability that $\ba$ is valid for the original instance $\bx$ is. 
Note that we can employ any distribution as $\mathcal{D}_{\tx}$ (e.g., the uniform distribution over $\mathcal{I}_{\tx}$ or a distribution estimated in advance~\citep{Little:2019:Missing}).

\section{Optimization framework}\label{sec:optimization}
This section proposes a practical optimization framework for \cref{prob:armin}. 
Unfortunately, evaluating the expected validity $V_{\tx}(\ba)$ is intractable because the imputation space $\mathcal{I}_{\tx}$ is an infinite set if the missing features are real-valued or an exponentially large set if they are categorical. 
To avoid this difficulty, we take the following two steps: 
(i)~estimating the expected validity $V_{\tx}(\ba)$ by sampling imputation candidates over $\mathcal{I}_{\tx}$; 
(ii)~solving a surrogate problem with the estimated validity by mixed-integer linear optimization. 
We also show some theoretical analyses of our proposed framework. 

\subsection{Imputation sampling}
Our idea is based on the fact that the constraint $V_{\tx}(\ba) \geq \rho$ of \cref{prob:armin} can be regarded as a \emph{chance constraint}~\citep{Calafiore:MP2005} over the distribution $\mathcal{D}_{\tx}$. 
As with the existing methods for chance-constrained problems, we estimate $V_{\tx}(\ba)$ by its tractable empirical average over a sample of imputation candidates. 

We first take an i.i.d.\ sample of $N$ imputation candidates $S = \set{ \hx_1, \dots, \hx_N } \subseteq \mathcal{I}_{\tx}$ from $\mathcal{D}_{\tx}$. 
Then, we replace the expected validity $V_{\tx}(\ba)$ in \cref{prob:armin} with its empirical average over $S$; that is, 
\begin{align*}
    \hat{V}_{\tx}(\ba \mid S) \coloneqq \frac{1}{N} {\sum}_{n=1}^{N} v(\ba; \hx_n),
\end{align*}
where $v(\ba; \hx) \coloneqq \ind{h(\hx + \ba) = +1}$ is the indicator whether the action $\ba$ is valid for $\hx$.

\myparagraph{On the sampling quality of $S$.}
The sample of imputation candidates $S$ does not necessarily include the original instance $\bx$. 
To guarantee that we can obtain at least one imputation candidate $\hx$ near to $\bx$ with high probability, we give a lower bound on the sampling size $N$ in \cref{prop:sample}. 
\begin{proposition}\label{prop:sample}
    We assume $\mathcal{X}_d = [l_d, u_d]$ for $d \in [D]$, where $w = u_d - l_d$ with a some width $w > 0$. 
    For $\tx \in \tilde{\mathcal{X}}$, let $S$ be a set of $N$ i.i.d.\ imputation candidates sampled from the uniform distribution. 
    Then, for any $\epsilon, \delta > 0$, there exists $\hx \in S$ such that $\| \hx - \bx \|_\infty \leq \epsilon$ with probability at least $1-\delta$ if $N \geq \log \frac{1}{\delta} \cdot \left( \frac{w}{2 \cdot \epsilon} \right)^{D_{\ast}}$, where $D_{\ast} = |M_{\tx}|$ is the number of missing features. 
\end{proposition}

\myparagraph{On the estimation quality of $\hat{V}_{\tx}(\ba \mid S)$.}
While the expected validity $V_{\tx}(\ba)$ is the probability that the action $\ba$ is valid over the imputation space $\mathcal{I}_{\tx}$, the empirical validity $\hat{V}_{\tx}(\ba \mid S)$ estimates it over a finite sample $S$ over $\mathcal{I}_{\tx}$. 
In \cref{prop:vc}, we present an upper bound on its estimation error from the perspective of the \emph{growth function}~\citep{Mohri:2012:Foundations} corresponding to a given classifier $h$ and action set $\calA$. 
\begin{proposition}\label{prop:vc}
    For a fixed classifier $h$ and action set $\calA$, let $\mathcal{H} \coloneqq \set{h_{\ba} \colon \mathcal{X} \to \mathcal{Y} \mid h_{\ba}(\bx) \coloneqq h(\bx + \ba), \ba \in \calA}$ be the family of functions that return the prediction value of $h$ after implementing an action $\ba$. 
    For $\mathcal{H}$ and $m > 0$, we denote its growth function by $\Pi_{\mathcal{H}}(m) = \max_{\set{\bx_1, \dots, \bx_m} \subseteq \mathcal{X}} |\set{(h_{\ba}(\bx_1), \dots, h_{\ba}(\bx_m)) \mid h_{\ba} \in \mathcal{H}}|$. 
    Then, for any $\ba \in \calA$ and $\delta > 0$, the following inequality holds with probability at least $1 - \delta$:
    \begin{align*}
        \hat{V}_{\tx}(\ba \mid S) - V_{\tx}(\ba) \leq \sqrt{\frac{2 \cdot \ln \Pi_{\mathcal{H}}(N)}{N}} + \sqrt{\frac{\ln \frac{1}{\delta}}{2 \cdot N}}. 
    \end{align*}
    Furthermore, we consider a linear classifier $h_{\bm{\beta}}(\bx) = \operatorname{sgn}(\bm{\beta}^\top \bx)$ with $\bm{\beta} = (\beta_1, \dots, \beta_D) \in \mathbb{R}^D$ and assume that $\calA$ is a convex set and contains at least one action $\ba$ such that $\ba \not= \bm{0}$. 
    Then, the following inequality holds with probability at least $1 - \delta$ for any $\ba \in \calA$ and $\delta > 0$:
    \begin{align*}
        \hat{V}_{\tx}(\ba \mid S) - V_{\tx}(\ba) \leq \sqrt{\frac{2 \cdot \ln (N + 1)}{N}} + \sqrt{\frac{\ln \frac{1}{\delta}}{2 \cdot N}}. 
    \end{align*}    
\end{proposition}
By definition, the growth function $\Pi_{\mathcal{H}}$ depends on the variety of the corresponding action set $\calA$. 
For $\calA$ and $\calA'$ such that $\calA \subseteq \calA'$, we have $\Pi_{\mathcal{H}}(N) \leq \Pi_{\mathcal{H}'}(N)$, where $\mathcal{H}$ and $\mathcal{H}'$ correspond to $\calA$ and $\calA'$, respectively. 
Thus, we should not make $\calA$ larger than necessary to reduce the estimation error.

\subsection{Mixed-integer linear optimization approach}\label{sec:optimization:milo}
Using the \emph{empirical validity} $\hat{V}_{\tx}(\ba \mid S)$, we can formulate our surrogate problem as
\begin{align}\label{eq:armin}
    {\mathop{\text{\upshape minimize}}}_{\ba \in \calA} \; c(\ba) \;\;\; \text{\upshape subject to} \;\; \hat{V}_{\tx}(\ba \mid S) \geq \rho. 
\end{align}
To solve the problem of \eqref{eq:armin}, we extend the existing AR methods based on mixed-integer linear optimization~(MILO)~\citep{Ustun:FAT*2019,Kanamori:AAAI2021,Parmentier:ICML2021}. 
Compared to the major gradient-based approaches, MILO-based approaches have several advantages:
(i)~we can apply them not only to a neural network but also to a tree ensemble, which is a non-differentiable model; 
(ii)~they can deal with several cost functions, including not only norm-based costs but also percentile-based costs; 
(iii)~they can naturally handle task-specific constraints (e.g., one-hot encoded categorical features); 
(iv)~we can obtain optimal solutions using off-the-shelf solvers, such as Gurobi, without implementing a designed algorithm. 

In the following, we propose a MILO formulation for a linear classifier $h_{\bm{\beta}}(\bx) = \operatorname{sgn}(\bm{\beta}^\top \bx)$ and a linear cost function $c(\ba) = \sum_{d=1}^{D} c_{d}(a_{d})$. 
Our MILO formulations for a neural network and a tree ensemble are presented in \cref{sec:appendix:formulation}. 
Note that several major cost functions, such as weighted $\ell_1$-norm~\citep{Wachter:HJLT2018} and total-log percentile shift~(TLPS)~\citep{Ustun:FAT*2019}, can be expressed as the above linear form and that our formulations can be extended to non-linear cost functions~\citep{Ustun:FAT*2019,Kanamori:AAAI2021}. 
As with the existing methods, we assume that each coordinate $\calA_d$ of a given action set $\calA = \calA_1 \times \dots \times \calA_D$ is finite and discretized; that is, we assume $\calA_d = \set{a_{d,1}, \dots, a_{d, J_d}}$, where $J_d = |\calA_d|$. 
To express an action $\ba \in \calA$, we introduce binary variables $\pi_{d,j} \in \set{0,1}$ for $d \in [D]$ and $j \in [ J_d ]$, which indicate that $a_{d,j} \in A_d$ is selected or not. 
We also introduce auxiliary variables $\nu_n \in \set{0,1}$ for $n \in [N]$ such that $\nu_n = v(\ba; \hx_n)$. 
Then, we can formulate the problem of \eqref{eq:armin} for the linear classifier $h_{\bm{\beta}}$ as follows:
\begin{align}\label{eq:linear}
\renewcommand{\arraystretch}{1.5}
    \begin{array}{cl}
        \text{minimize}   & \sum_{d=1}^{D} \sum_{j=1}^{J_d} c_{d}(a_{d,j}) \cdot \pi_{d,j} \\
        \text{subject to} & \sum_{n=1}^{N} \nu_n \geq N \cdot \rho, \\
                          & \sum_{j=1}^{J_d} \pi_{d,j} = 1, \forall d \in [D], \\
                          & \bm{\beta}^\top \hx_{n} + \sum_{d=1}^{D} \beta_d \cdot \sum_{j=1}^{J_d} a_{d,j} \cdot \pi_{d,j} \geq M \cdot (1 - \nu_n),  \forall n \in [N], \\
                          & \pi_{d,j} \in \set{0,1}, \forall d \in [D], \forall j \in [J_d], \\
                          & \nu_{n} \in \set{0,1}, \forall n \in [N],
    \end{array}
\renewcommand{\arraystretch}{1.0}
\end{align}
where $M$ is a constant such that $M \leq \min_{\hx \in S} \min_{\ba \in \calA} \bm{\beta}^\top (\hx + \ba)$. 
Note that we can compute $c_{d}(a_{d,j})$ for each $a_{d,j}$ as a constant when $\calA$ is given. 
Since the problem of \eqref{eq:linear} is a MILO problem, we can solve it using MILO solvers and recover an optimal action to \eqref{eq:armin} from the obtained solution.

\myparagraph{Heuristic solution approach.}
While we can formulate \eqref{eq:armin} as a MILO problem by extending the existing MILO-based AR methods, solving the formulated problem is challenging for a large $N$ and a complex classifier $h$ since the total number of the constraints increases depending on $N$ and $h$. 
Indeed, in our preliminary experiments, we often failed to obtain low-cost actions within a given time limit. 
To alleviate this issue, in \cref{sec:appendix:implementation:armin}, we propose a heuristic approach that divides the problem into a set of a few small problems by subsampling the imputation candidates. 
We confirmed that it could improve the quality of actions without increasing the overall computational time.

\begin{figure}[t]
    \centering
    \includegraphics[width=\textwidth]{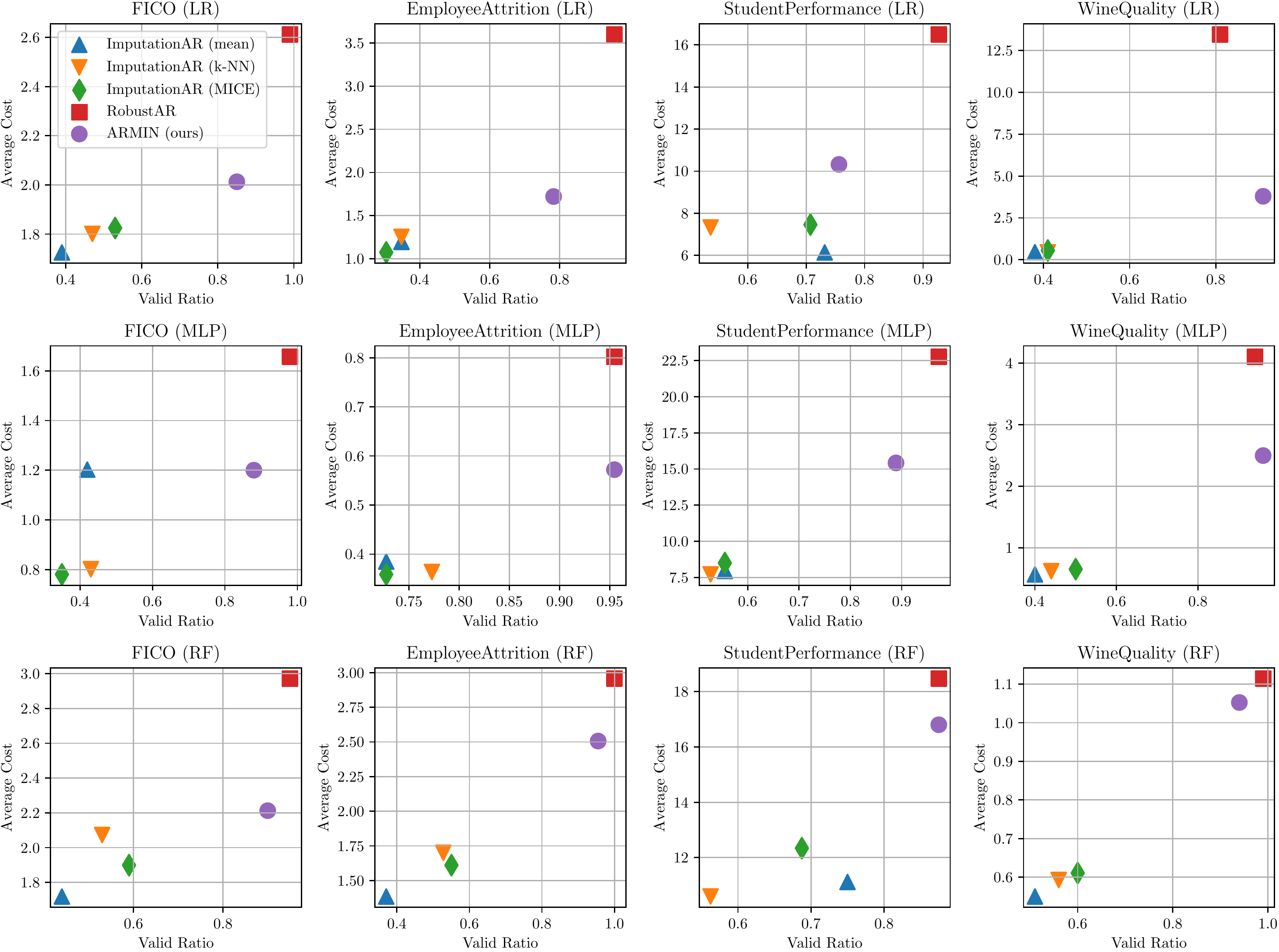}
    \caption{
        Experimental results of our baseline comparison under the MCAR situation, where $D_{\ast} = 2$. 
        The x-axis (resp.\ y-axis) represents the valid ratio (resp.\ average cost). 
        Compared to the baselines, our ARMIN attained good balances between the valid ratio and cost for almost all the datasets and classifiers; that is, it achieved higher valid ratios than ImputationAR and lower costs than RubustAR. 
    }
    \label{fig:exp:mcar}
\end{figure}

\section{Experiments}\label{sec:experiments}
To investigate the efficacy of our ARMIN, we conducted experiments on public datasets. 
The code was implemented in Python 3.10 with scikit-learn 1.0.2 and Gurobi 10.0.3. 
The experiments were conducted on Ubuntu 20.04 with Intel Xeon E-2274G 4.0 GHz CPU and 32 GB memory. 

Our experimental evaluation aims to answer the following questions:
(i)~How are the validity and cost of the recourse actions obtained by our ARMIN compared to those by the baselines? 
(ii)~How does our ARMIN behave under more realistic missing mechanisms compared to the baselines? 
(iii)~How is the trade-off between the validity and cost depending on our confidence parameter $\rho$? 
Due to page limitations, our implementation details and the complete results are shown in \cref{sec:appendix:implementation,sec:appendix:experiments}. 

\myparagraph{Experimental settings. }
We used FICO ($D=23$)~\citep{fico:2018}, Attrition ($D=44$)~\citep{Attrition:2017}, WineQuality ($D=12$), Student ($D=48$)~\citep{Dua:2019}, and GiveMeCredit ($D=10$)~\citep{Givemecredit:2011} datasets. 
We randomly split each dataset into train (75\%) and test (25\%) instances, and trained a logistic regression classifier (LR), two-layer ReLU network with $30$ neurons (MLP), and random forest with $50$ decision trees (RF) as $h$ on each training dataset. 
To simulate the situation where instances include missing values, for each test instance $\bx$ with the undesired prediction $h(\bx) = -1$ (e.g., ``loan rejection"), we generated its incomplete instances $\tx$ by dropping their values of some features. 
Then, we extracted actions for $\tx$ by each method. 
As a cost function $c$, we used the TLPS~\citep{Ustun:FAT*2019} $c(\ba) = {\sum}_{d=1}^{D} \ln \left( \frac{1-Q_d(x_d + a_d)}{1-Q_d(x_d)} \right)$, where $Q_d$ is the cumulative distribution function of $d$. 
For our ARMIN, we estimated $\mathcal{D}_{\tx}$ in advance by MICE with the Bayesian ridge implemented in scikit-learn, and set $N=100$ and $\rho = 0.75$. 

\myparagraph{Baselines.}
To the best of our knowledge, no existing AR method directly works with missing values. 
Thus, we consider two baselines by extending the existing methods. 
One is \emph{ImputationAR}, which applies an existing AR method to the imputed instance $\hx$ by an existing imputation method including mean imputation, $k$-NN imputation~\citep{Troyanskaya:Bioinformatics2001}, and MICE~\citep{Buuren:JSS2011}. 
The other baseline is \emph{RobustAR}~\citep{Dominguez-olmedo:ICML2022}, where we replace its uncertainty set with a set of $N$ imputation candidates sampled over ${\mathcal{I}}_{\tx}$.

\subsection{Comparison under MCAR situation}
Firstly, we examined each method under the MCAR situation. 
To simulate the MCAR mechanism, we randomly selected a few features for each test instance $\bx$ and dropped its values of the selected features. 
We set the total number of missing features $D_{\ast} \in \set{1,2,3}$. 
We measured the \emph{valid ratio}, which is the ratio that an obtained action $\ba$ is valid for the original instance $\bx$ (i.e., $h(\bx+\ba) = +1$) and the average cost $c(\ba)$. 
In \cref{sec:appendix:experiments}, we also report the average \emph{sign agreement score}~\citep{Krishna:arxiv2022} that evaluates the qualitative similarity between $\ba$ and the optimal action $\ba^\ast$ of $\bx$ and the average computational time for each instance. 

\cref{fig:exp:mcar} shows the results of the valid ratio and average cost for $D_{\ast} = 2$. 
From these results, we can see that \emph{our ARMIN attained good balances between the valid ratio and average cost than the baselines}. 
We observe that the valid ratio of ImputationAR was significantly lower than RobustAR and ARMIN, and that RobustAR and ARMIN stably achieved high valid ratios regardless of the datasets and classifiers. 
In addition, the average cost of RobustAR was always larger than ARMIN. 
For example, the average cost of RobustAR for the LR classifier on the WineQuality dataset was 3.59 times higher than that of ARMIN. 
In summary, we confirmed that \emph{our ARMIN could provide given incomplete instances with good actions that achieve higher validity and lower cost than the baselines}.

\begin{wrapfigure}[14]{r}{0.5\textwidth}
    \vspace{-2mm}
    \centering
    \includegraphics[width=0.5\textwidth]{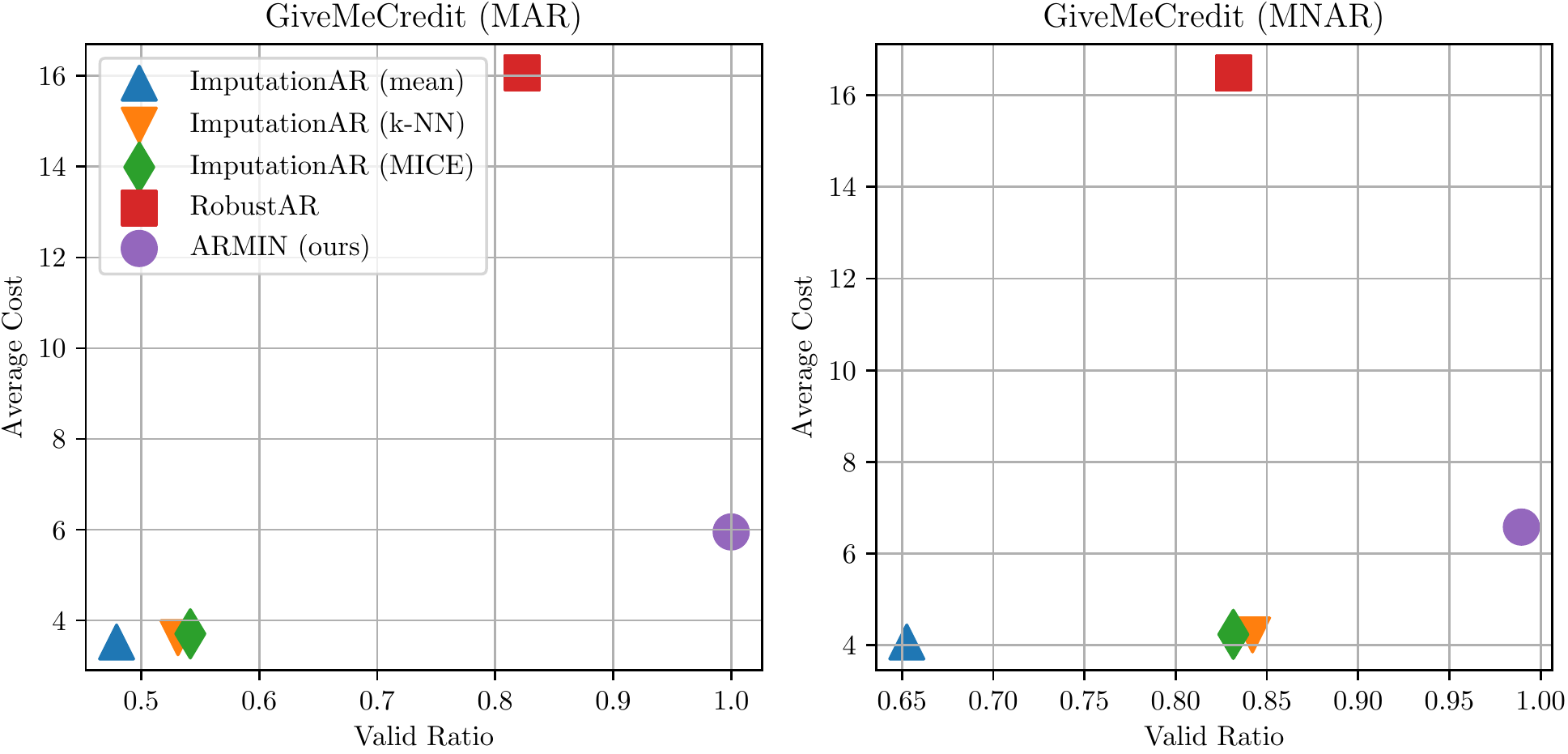}
    \vspace{-4mm}
    \caption{
        Experimental results of our comparison under the MAR and MNAR situations. 
    }
    \label{fig:exp:marmnar}
\end{wrapfigure}
\subsection{Comparison under MAR and MNAR situations}
Next, we examined each method on the GiveMeCredit dataset under the MAR and MNAR situations. 
To simulate the MAR (resp.\ MNAR) mechanism, we assume the situation where older (resp.\ richer) people are less inclined to reveal their income~\citep{Josse:arxiv2019}. 
We collected test instances $\bx$ that were predicted as ``experiencing 90 days past due delinquency or worse" by an LR classifier $h$ and older than the median of the feature ``Age" (resp.\ had larger incomes than the median of the feature ``MonthlyIncome"). 
Then, we dropped their values of ``MonthlyIncome," and obtained actions $\ba$ for them. 

\begin{table}[t]
    \centering
    \caption{
        Examples of the obtained actions $\ba$ by the baselines and our ARMIN on the GiveMeCredit dataset. 
        Here, Valid and Cost indicate $h(x+a) = +1$ and $c(a)$, respectively. 
        While (a) is the optimal action for a given instance $\bx$ without missing value, (b), (c), and (d) are the actions obtained by the baselines and our method for the incomplete instance $\tx$, where the feature ``MonthlyIncome" is missing following the MAR mechanism. 
        Compared to the baselines, our ARMIN could obtain the same action as the optimal one, though the given instance included missing values. 
    }
    \subfigure[Optimal action without missing values]{
        \small
        \adjustbox{valign=b}{
            \begin{tabular}{lccc}
            \toprule
                \textbf{Feature} & \textbf{Value}  & \textbf{Valid}  & \textbf{Cost} \\
            \midrule
               NumberOfTimes90DaysLate & $-3$ & True & $3.00$ \\
            \bottomrule
            \end{tabular}
            \label{tab:exp:marexample:optimal}
        }
    }
    \hfill
    \subfigure[ImputationAR (mean)]{
        \small
        \adjustbox{valign=b}{
            \begin{tabular}{lccc}
            \toprule
                \textbf{Feature} & \textbf{Value}  & \textbf{Valid}  & \textbf{Cost} \\
            \midrule
               RevolvingUtilizationOfU.L. & $-0.49$ & False & $2.75$ \\
            \bottomrule
            \end{tabular}
            \label{tab:exp:marexample:mean}
        }
    }
    \hfill
    \subfigure[RobustAR]{
        \small
        \adjustbox{valign=b}{
            \begin{tabular}{lccc}
            \toprule
                \textbf{Feature} & \textbf{Value}  & \textbf{Valid}  & \textbf{Cost} \\
            \midrule
               RevolvingUtilizationOfU.L. & $-0.36$ & \multirow{2}{*}{True} & \multirow{2}{*}{$5.42$} \\
               NumberOfTimes90DaysLate & $-3$ \\
            \bottomrule
            \end{tabular}
            \label{tab:exp:marexample:robust}
        }
    }
    \hfill
    \subfigure[ARMIN (ours)]{
        \small
        \adjustbox{valign=b}{
            \begin{tabular}{lccc}
            \toprule
                \textbf{Feature} & \textbf{Value}  & \textbf{Valid}  & \textbf{Cost} \\
            \midrule
               NumberOfTimes90DaysLate & $-3$ & True & $3.00$ \\
            \bottomrule
            \end{tabular}
            \label{tab:exp:marexample:armin}
        }
    }
    \label{tab:exp:marexample}
\end{table}

\cref{fig:exp:marmnar} shows the results of the valid ratio and average cost for the baselines and our method. 
Similar to the MCAR setting, we can see that the valid ratio of ImputationAR was significantly lower than the others, and the costs of RobustAR were higher than ARMIN. 
These results indicate that \emph{our ARMIN could yield better actions than the baselines under the MAR and MNAR situations as well}. 

\cref{tab:exp:marexample} presents examples of the obtained actions for the same instance under the MAR situation. 
While \cref{tab:exp:marexample:optimal} is the optimal action for the original instance $\bx$ without missing value, \cref{tab:exp:marexample:mean,tab:exp:marexample:robust,tab:exp:marexample:armin} are the actions obtained by the baselines and our method for its incomplete instance $\tx$. 
We observe that the action obtained by ImputationAR is different from the optimal one in \cref{tab:exp:marexample:optimal} and not valid for $\bx$.
We also see that while the action obtained by RobustAR is valid for $\bx$, it has a significantly higher cost than that of the optimal one. 
In contrast, ARMIN could obtain the same action as the optimal one.
That is, \emph{our ARMIN succeeded in providing a good action that is valid and low-cost for the original instance $\bx$, even though the given instance included missing values}.

\begin{figure}[t]
    \centering
    \hfill
    \subfigure[Sensitivity analysis of the confidence parameter $\rho$]{
        \includegraphics[width=0.475\textwidth,valign=b]{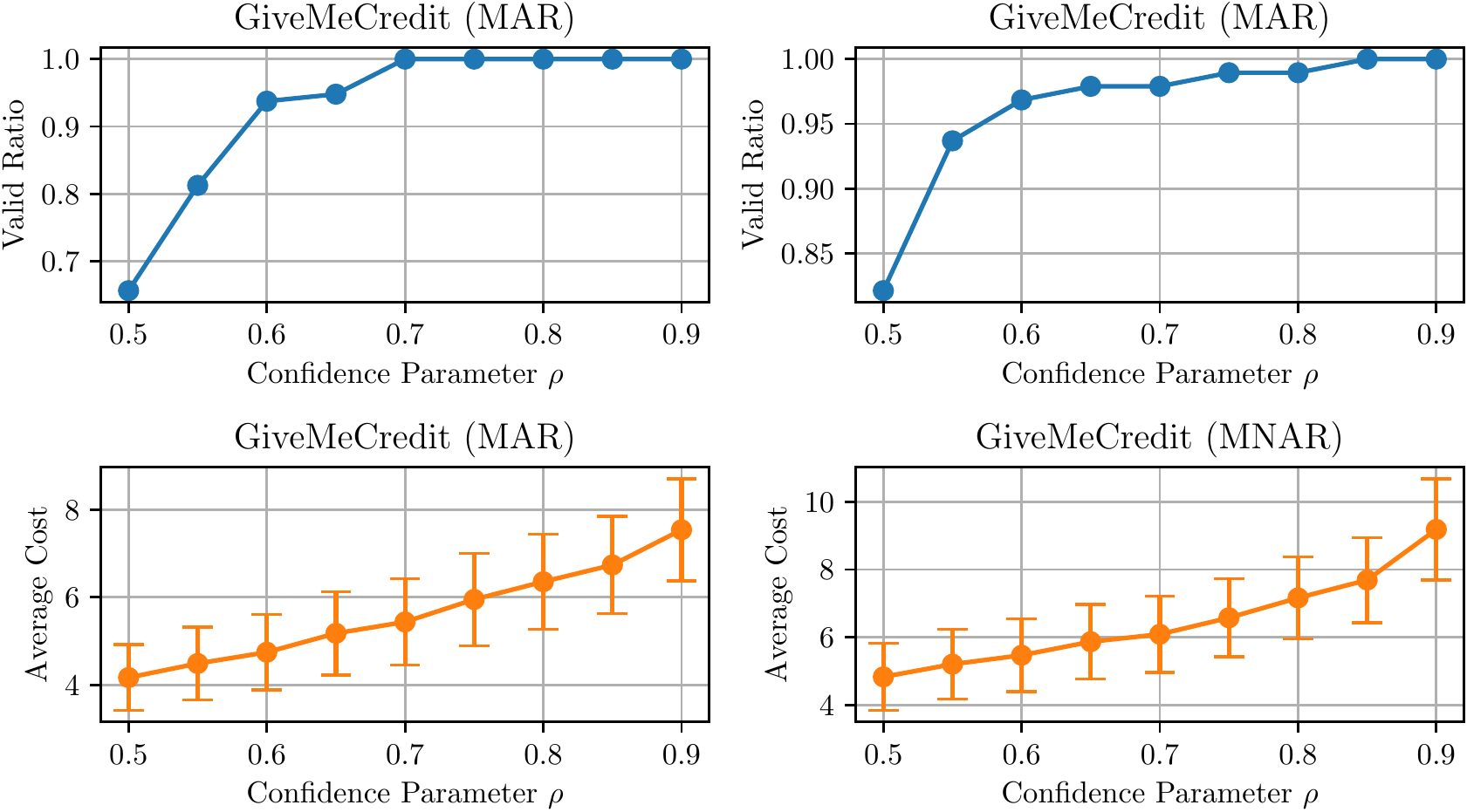}
        \label{fig:exp:sens}
    }
    \hfill
    \subfigure[Path analysis of the confidence parameter $\rho$]{
        \includegraphics[width=0.475\textwidth,valign=b]{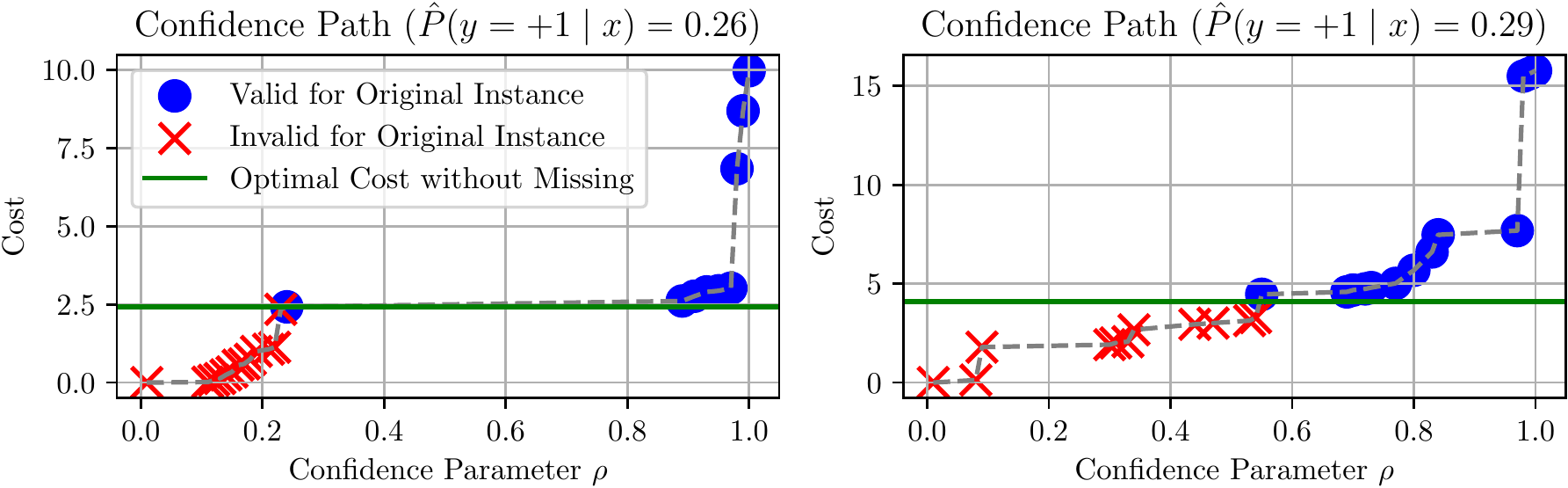}
        \label{fig:exp:path}
    }
    \caption{
        Experimental results of our trade-off analyses on the GiveMeCredit dataset. 
        (a)~Experimental results of the sensitivity analyses of the confidence parameter $\rho$ with $95\%$ confidence intervals. 
        (b)~Experimental results of the path analyses under the MAR situation. 
        Here, we selected two examples, and additional examples are shown in \cref{sec:appendix:experiments}. 
    }
\end{figure}

\subsection{Analysis of trade-off between validity and cost}
Finally, we analyze the trade-off between the validity and cost of our ARMIN by varying its confidence parameter $\rho \in \set{0.5, 0.55, \dots, 0.9}$. 
We show the results on the GiveMeCredit dataset under the MAR and MNAR situation. 
\cref{fig:exp:sens} shows the results of the valid ratio and average cost for each $\rho$. 
We observed that while both the valid ratio and average cost are increased by increasing $\rho$, the increase of the valid ratio was almost saturated when $\rho \geq 0.6$. 
This result suggests that \emph{there exists an appropriate value of $\rho$ that makes the obtained action $\ba$ valid for the original instance $\bx$ at high probability without increasing its cost as much as possible}.

For a more detailed analysis of the trade-off, we also measured the validity and cost for each original instance by our path algorithm presented in \cref{algo:path} of Appendix. 
\cref{fig:exp:path} presents two examples of our results. 
We can see that the confidence $\rho$ that makes the obtained action $\ba$ valid for the original instance $\bx$ varies depending on each instance. 
We also observed that the costs increased sharply when the confidences $\rho$ exceed $0.9$ and close to $1$, which may be related to our results in \cref{fig:exp:mcar,fig:exp:marmnar} that the average cost of RobustAR is higher than the others.

\section{Conclusion}\label{sec:concl}
This paper proposed the first framework of algorithmic recourse (AR) that works even in the presence of missing values, named algorithmic recourse with multiple imputation~(ARMIN). 
We first discussed the problem of AR with missing values. 
We empirically and theoretically showed the risk that a naive approach with a single imputation technique often fails to obtain good actions regarding their validity, cost, and features to be changed. 
To alleviate this risk, we formulated the task of obtaining a valid and low-cost action for a given incomplete instance. 
Our main idea is to incorporate well-known multiple imputation techniques into the AR optimization problem. 
Then, we provided theoretical analyses of our task and proposed a practical solution based on mixed-integer linear optimization. 
Our experiments on public datasets demonstrated the efficacy of ARMIN compared to existing methods.

\myparagraph{Limitations and future work.}
There are several directions to improve our ARMIN. 
First, our experiments demonstrated that while our ARMIN could obtain better actions, it was certainly slower than the baselines. 
Therefore, we will try to improve the efficiency of our optimization framework, e.g., by leveraging the techniques of chance-constrained optimization algorithms~\citep{Calafiore:MP2005}. 
In addition, we observed that an appropriate confidence parameter $\rho$ varies depending on each instance $\bx$. 
Because running our method repeatedly by varying $\rho$ is costly in practice, we need to develop a method for determining its default value automatically. 
Furthermore, it is important to show theoretical relationships between $\rho$ and the expected validity of the obtained action for the original instance. 
Finally, extending our framework to other explanation methods, such as LIME~\citep{Ribeiro:KDD2016}, is also interesting.

\bibliographystyle{abbrvnat}
\begin{small}
    \bibliography{ref}
\end{small}

\newpage
\appendix

\section{Omitted proof}\label{sec:appendix:proof}
\subsection{Proof of \texorpdfstring{\cref{prop:upper}}{}}
To prove \cref{prop:upper}, we use the following lemma by~\citep{Ustun:FAT*2019,Pawelczyk:AISTATS2022}. 
Recall that we are given 
(i) a classifier $h_{\bm{\beta}}(\bx) = \operatorname{sgn}(f_{\bm{\beta}}(\bx))$ with a linear decision function $f_{\bm{\beta}}(\bx) = \bm{\beta}^\top \bx$ and a parameter $\bm{\beta} \in \mathbb{R}^D$, 
(ii) a feasible action set $\calA = \mathbb{R}^D$ for any $\bx \in \mathcal{X}$, and 
(iii) a cost function $c(a) = \| \ba \|$. 

\begin{lemma}[Closed-form Optimal Action~\citep{Ustun:FAT*2019,Pawelczyk:AISTATS2022}]\label{lemm:appendix:closedform}
    For a given instance $\bx \in \mathcal{X}$ without missing values, let $\ba^\ast(\bx)$ be an optimal solution to the problem~(2) for $\bx$. 
    Then, we have 
    \begin{align}
        \ba^\ast(\bx) = 
        \begin{cases}
            -\frac{f_{\bm{\beta}}(\bx)}{\| \bm{\beta} \|_2^2} \bm{\beta} & \text{if } h_{\bm{\beta}}(\bx)=-1, \\
            \bm{0} & \text{otherwise}.
        \end{cases}
    \end{align}
\end{lemma}

Using \cref{lemm:appendix:closedform}, we give a proof of \cref{prop:upper} as follows. 
As mentioned in the main paper, we assume that a feature $d^\circ \in [D]$ is missing and imputed with its population mean $\mu_{d^\circ} = \mathbb{E}_{\bx} \left[ x_{d^\circ} \right]$. 
\begin{proof}[Proof of \cref{prop:upper}]
    Recall that $\hat{x}_{d^\circ} = \mu_{d^\circ}$ and $\hat{x}_d = x_d$ for $d \in [D] \setminus \set{d^\circ}$. 
    From \cref{lemm:appendix:closedform}, we obtain
    \begin{align*}
        \| \hat{\ba} - \ba^\ast \|_2^2 
        =
        \begin{cases}
            \frac{ \beta_{d^\circ}^2}{\| \bm{\beta} \|_2^2} \cdot (x_{d^\circ} - \mu_{d^{\circ}})^2 & \text{if } h_{\bm{\beta}}(\bx) = -1 \land h_{\bm{\beta}}(\hx) = -1, \\
            \frac{1}{\| \bm{\beta} \|_2^2} \cdot (f_{\bm{\beta}}(\bx))^2 & \text{if } h_{\bm{\beta}}(\bx) = -1 \land h_{\bm{\beta}}(\hx) = +1, \\
            \frac{1}{\| \bm{\beta} \|_2^2} \cdot (f_{\bm{\beta}}(\hx))^2 & \text{if } h_{\bm{\beta}}(\bx) = +1 \land h_{\bm{\beta}}(\hx) = -1, \\
            \bm{0} & \text{otherwise}.
        \end{cases}
    \end{align*}
    Thus, we have
    \begin{align*}
        &\mathbb{E}_{\bx} \left[ \| \hat{\ba} - \ba^\ast \|_2^2 \right]  \\
        &= \frac{\beta_{d^\circ}^2}{\| \bm{\beta} \|_2^2} \cdot \mathbb{E}_{\bx}\left[  (x_{d^\circ} - \mu_{d^{\circ}})^2 \mid h_{\bm{\beta}}(\bx) = -1 \land h_{\bm{\beta}}(\hx) = -1 \right] \cdot \mathbb{P}_{\bx}\left[ h_{\bm{\beta}}(\bx) = -1 \land h_{\bm{\beta}}(\hx) = -1 \right] \\
        &+ \frac{1}{\| \bm{\beta} \|_2^2} \cdot \mathbb{E}_{\bx}\left[ \left(f_{\bm{\beta}}(\bx)\right)^2 \mid h_{\bm{\beta}}(\bx) = -1 \land h_{\bm{\beta}}(\hx) = +1 \right] \cdot \mathbb{P}_{\bx}\left[ h_{\bm{\beta}}(\bx) = -1 \land h_{\bm{\beta}}(\hx) = +1 \right] \\
        &+ \frac{1}{\| \bm{\beta} \|_2^2} \cdot \mathbb{E}_{\bx}\left[ \left(f_{\bm{\beta}}(\hx)\right)^2 \mid h_{\bm{\beta}}(\bx) = +1 \land h_{\bm{\beta}}(\hx) = -1 \right] \cdot \mathbb{P}_{\bx}\left[ h_{\bm{\beta}}(\bx) = +1 \land h_{\bm{\beta}}(\hx) = -1 \right].
    \end{align*}

    For the case of $h_{\bm{\beta}}(\bx) = +1 \land h_{\bm{\beta}}(\hx) = -1$, we have
    \begin{align*}
        (f_{\bm{\beta}}(\hx))^2
        &= \left( f_{\bm{\beta}}({x}) - \beta_{d^\circ} (x_{d^\circ} - \mu_{d^{\circ}}) \right)^2 \\
        &= \left(f_{\bm{\beta}}(\bx)\right)^2 + \beta_{d^\circ}^2 (x_{d^\circ} - \mu_{d^{\circ}})^2 -2 \cdot f_{\bm{\beta}}(\bx) \cdot \beta_{d^\circ} (x_{d^\circ} - \mu_{d^{\circ}}) .
    \end{align*}
    Since $h_{\bm{\beta}}(\bx) = +1$ and $h_{\bm{\beta}}(\hx) = -1$, $f_{\bm{\beta}}(\bx) \geq 0$ and $f_{\bm{\beta}}(\bx) \geq f_{\bm{\beta}}(\hx) \iff \beta_{d^\circ} (x_{d^\circ} - \mu_{d^{\circ}}) \geq 0$ hold. 
    Thus, we obtain 
    \begin{align*}
        \left(f_{\bm{\beta}}(\hx)\right)^2 \leq \left(f_{\bm{\beta}}(\bx)\right)^2 + \beta_{d^\circ}^2 (x_{d^\circ} - \mu_{d^{\circ}})^2.
    \end{align*}

    By combining the above results, we have
    \begin{align*}
        \mathbb{E}_{\bx} \left[ \| \hat{\ba} - \ba^\ast \|_2^2 \right] \leq &\frac{\beta_{d^\circ}^2}{\| \bm{\beta} \|_2^2} \cdot \mathbb{E}_{\bx}\left[  (x_{d^\circ} - \mu_{d^{\circ}})^2 \mid h_{\bm{\beta}}(\bx) = -1 \right] \cdot \mathbb{P}_{\bx}\left[ h_{\bm{\beta}}(\bx) = -1 \right] \\
        &+ \frac{1}{\| \bm{\beta} \|_2^2} \cdot \mathbb{E}_{\bx}\left[ \left(f_{\bm{\beta}}(\bx)\right)^2 \mid h_{\bm{\beta}}(\bx) \not= h_{\bm{\beta}}(\hx) \right] \cdot \mathbb{P}_{\bx}\left[ h_{\bm{\beta}}(\bx) \not= h_{\bm{\beta}}(\hx) \right] \\
        \leq &\frac{\beta_{d^\circ}^2}{\| \bm{\beta} \|_2^2} \cdot \mathbb{E}_{\bx}\left[  (x_{d^\circ} - \mu_{d^{\circ}})^2 \right] + \frac{1}{\| \bm{\beta} \|_2^2} \cdot \gamma \cdot p_\mathrm{conf},
    \end{align*}
    where $\gamma = \mathbb{E}_{\bx}[ \left( f_{\bm{\beta}}(\bx) \right)^2 \mid h_{\bm{\beta}}(\bx) \not= h_{\bm{\beta}}(\hx)]$ and $p_\mathrm{conf} = \mathbb{P}_{\bx}[h_{\bm{\beta}}(\bx) \not= h_{\bm{\beta}}(\hx)]$. 
    The final inequality holds because $(x_{d^\circ} - \mu_{d^{\circ}})^2 \geq 0$ and 
    \begin{align*}
        \mathbb{E}_{\bx}\left[  (x_{d^\circ} - \mu_{d^{\circ}})^2 \right] 
        = &\mathbb{E}_{\bx}\left[  (x_{d^\circ} - \mu_{d^{\circ}})^2 \mid h_{\bm{\beta}}(\bx) = -1 \right] \cdot \mathbb{P}_{\bx}\left[ h_{\bm{\beta}}(\bx) = -1 \right] \\
        &+ \mathbb{E}_{\bx}\left[  (x_{d^\circ} - \mu_{d^{\circ}})^2 \mid h_{\bm{\beta}}(\bx) = +1 \right] \cdot \mathbb{P}_{\bx}\left[ h_{\bm{\beta}}(\bx) = +1 \right] \\
        \geq &\mathbb{E}_{\bx}\left[  (x_{d^\circ} - \mu_{d^{\circ}})^2 \mid h_{\bm{\beta}}(\bx) = -1 \right] \cdot \mathbb{P}_{\bx}\left[ h_{\bm{\beta}}(\bx) = -1 \right].
    \end{align*}
    

    Since $\sigma^2_{d^\circ} = \mathbb{E}_{\bx}\left[  (x_{d^\circ} - \mu_{d^{\circ}})^2 \right]$, we obtain
    \begin{align*}
        \mathbb{E}_{\bx} \left[ \| \hat{\ba} - \ba^\ast \|_2^2 \right] \leq \frac{1}{\| \bm{\beta} \|_2^2} \left( \beta_{d^\circ}^2 \cdot \sigma_{d^\circ}^{2} + \gamma \cdot p_\mathrm{conf} \right). 
    \end{align*}
\end{proof}

By a similar argument, we obtain an upper bound on the expected difference $\mathbb{E}_{x} \left[ \| \hat{\ba} - \ba^\ast \|_2^2 \right]$ for a more general imputation method in the following proposition. 
\cref{prop:upper} can be regarded as a special case of the following.

\begin{proposition}
    Let $i_{d^\circ} \colon \tilde{\mathcal{X}} \to \mathcal{X}_{d^\circ}$ be an imputation function for the feature $d^\circ$. 
    For an instance $\bx \in \mathcal{X}$, we denote its incomplete instance $\tx \in \tilde{\mathcal{X}}$ (resp.\ its imputed instance $\hx \in \mathcal{X}$) with $\tilde{x}_{d^\circ} = *$ (resp.\ $\hat{x}_{d^\circ} = i_{d^\circ}(\tx)$) and $\tilde{x}_d = x_d$ (resp.\ $\hat{x}_d = x_d$) for $d \in [D] \setminus \set{d^\circ}$. 
    Then, we have
    \begin{align}
        \mathbb{E}_{\bx} \left[ \| \hat{\ba} - \ba^\ast \|_2^2 \right] 
        \leq 
        \frac{1}{\| \bm{\beta} \|_2^2}
        \left(
        \beta_{d^\circ}^2 \cdot \gamma_\mathrm{loss}
        + 
        \gamma \cdot p_\mathrm{conf}
        \right), 
    \end{align}
    where $\gamma_\mathrm{loss} = \mathbb{E}_{\bx} \left[ (x_{d^\circ} - i_{d^\circ}(\tx))^2 \right]$ is the expected squared loss of the imputation function $i_{d^\circ}$. 
    Furthermore, if $i_{d^\circ}$ satisfies $\mathbb{E}_{\bx} \left[ i_{d^\circ}(\tx) \right] = \mu_{d^\circ}$, then we have
    \begin{align}
        \mathbb{E}_{\bx} \left[ \| \hat{\ba} - \ba^\ast \|_2^2 \right] 
        \leq 
        \frac{1}{\| \bm{\beta} \|_2^2}
        \left(
        \beta_{d^\circ}^2 \cdot (\sigma_{d^\circ}^{2} + \gamma_\mathrm{cov})
        + 
        \gamma \cdot p_\mathrm{conf}
        \right), 
    \end{align}
    where $\gamma_\mathrm{cov} = \sigma_\mathrm{imp}^{2} - 2 \cdot \operatorname{Cov}(x_{d^\circ}, i_{d^\circ}(\tx))$,
    $\sigma_\mathrm{imp}^{2} = \mathbb{E}_{\bx} \left[ (i_{d^\circ}(\tx) - \mu_{d^\circ})^2 \right]$ is the variance of $i_{d^\circ}(\tx)$ on $\mathcal{D}_{\bx}$, and
    $\operatorname{Cov}(x_{d^\circ}, i_{d^\circ}(\tx)) = \mathbb{E}_{\bx} \left[ (x_{d^\circ} - \mu_{d^\circ})(i_{d^\circ}(\tx) - \mu_{d^\circ}) \right]$ is the covariance between $x_{d^\circ}$ and $i_{d^\circ}(\tx)$ on $\mathcal{D}_{\bx}$.
\end{proposition}
\begin{proof}
    As with the proof of Theorem~4, we obtain
    \begin{align*}
        \mathbb{E}_{\bx} \left[ \| \hat{\ba} - \ba^\ast \|_2^2 \right] 
        &\leq \frac{1}{\| \bm{\beta} \|_2^2} \left( \beta_{d^\circ}^2 \cdot \mathbb{E}_{\bx}\left[  (x_{d^\circ} - i_{d^\circ}(\tx))^2 \right] + \gamma \cdot p_\mathrm{conf} \right) \\
        &= \frac{1}{\| \bm{\beta} \|_2^2} \left( \beta_{d^\circ}^2 \cdot \gamma_\mathrm{loss} + \gamma \cdot p_\mathrm{conf} \right).
    \end{align*}
    For the term $\gamma_\mathrm{loss}$, if $\mathbb{E}_{\bx} \left[ i_{d^\circ}(\tx) \right] = \mu_{d^\circ}$ holds, we also obtain
    \begin{align*}
        \gamma_\mathrm{loss}
        &= \mathbb{E}_{\bx} \left[ (x_{d^\circ} - i_{d^\circ}(\tx))^2 \right] \\
        &= \mathbb{E}_{\bx} \left[ (x_{d^\circ} - i_{d^\circ}(\tx) - \mu_{d^\circ} + \mu_{d^\circ} )^2 \right] \\
        &= \mathbb{E}_{\bx} \left[ (x_{d^\circ} - \mu_{d^\circ})^2 \right] + \mathbb{E}_{\bx} \left[ (i_{d^\circ}(\tx) - \mu_{d^\circ})^2 \right] -2 \mathbb{E}_{\bx} \left[ (x_{d^\circ} - \mu_{d^\circ})(i_{d^\circ}(\tx) - \mu_{d^\circ}) \right] \\
        &= \sigma_{d^\circ}^{2} + \sigma_\mathrm{imp}^{2} - 2 \cdot \operatorname{Cov}(x_{d^\circ}, i_{d^\circ}(\tx)) = \sigma_{d^\circ}^{2} + \gamma_\mathrm{cov},
    \end{align*}
    which concludes the proof. 
\end{proof}

\subsection{Proof of \texorpdfstring{\cref{prop:sample}}{}}
\begin{proof}[Proof of \cref{prop:sample}]
    We consider to bound the probability that $\| \hx - \bx \|_\infty > \varepsilon$ holds for any $\hx \in S$ by $\delta$. 
    Recall that $S$ is a set of $N$ i.i.d.\ imputation candidates sampled over the uniform distribution on the imputation space $\mathcal{I}_{\tx}$. 
    Since $\| \hx - \bx \|_\infty > \varepsilon \iff \exists d \in M_{\tx}: |\hat{x}_d - x_d| > \varepsilon$ holds and $\mathbb{P} \left[ | \hat{x}_d - x_d | \leq \varepsilon \right] = \frac{2 \cdot \varepsilon}{w}$ for all $d \in [D]$, we have
    \begin{align*}
        \mathbb{P} \left[ \bigcap_{\hx \in S} ( \| \hx - \bx \|_\infty > \varepsilon ) \right]
        &= \mathbb{P} \left[ \bigcap_{\hx \in S} \left( \bigcup_{d \in M_{\tx})} (| \hat{x}_d - x_d | > \varepsilon) \right) \right] & \\
        &= \left( \mathbb{P} \left[ \bigcup_{d \in M_{\tx}} (| \hat{x}_d - x_d | > \varepsilon) \right] \right)^N\\
        &= \left( 1 - \prod_{d \in M_{\tx}} \mathbb{P} \left[ | \hat{x}_d - x_d | \leq \varepsilon \right] \right)^N\\
        &= \left( 1 - \left( \frac{2 \cdot \varepsilon}{w} \right)^{D_{\ast}} \right)^N 
        \leq \exp \left( -N \cdot  \left( \frac{2 \cdot \varepsilon}{w} \right)^{D_{\ast}} \right),
    \end{align*}
    where $\tilde{D} = |M_{\tx}|$. 
    Therefore, we obtain
    \begin{align*}
         \exp \left( -N \cdot  \left( \frac{2 \cdot \varepsilon}{w} \right)^{D_{\ast}} \right) \leq \delta
         \iff N \geq \log \frac{1}{\delta} \cdot \left( \frac{w}{2 \cdot \epsilon} \right)^{D_{\ast}}. 
    \end{align*}
\end{proof}

\subsection{Proof of \texorpdfstring{\cref{prop:vc}}{}}
First, we present an upper bound on the estimation error of the validity from the perspective of the \emph{Rademacher complexity}~\citep{Mohri:2012:Foundations} corresponding to a given classifier $h$ and action set $\calA$. 
For a family $\mathcal{G}$ of functions $g \colon \mathcal{Z} \to [a, b]$ and a set $Z = \set{z_1, \dots z_m} \subseteq \mathcal{Z}$, its empirical Rademacher complexity is defined by $\hat{\mathcal{R}}_Z(\mathcal{G}) \coloneqq \mathbb{E}_{\sigma \in \set{\pm 1}^m} [\sup_{g \in \mathcal{G}} \frac{1}{m} \sum_{i=1}^{m} \sigma_{i} \cdot g(z_i)]$. 
Using the empirical Rademacher complexity, we can define the Rademacher complexity by $\mathcal{R}_m(\mathcal{G}) \coloneqq \mathbb{E}_{Z}[\hat{\mathcal{R}}_Z(\mathcal{G})]$. 
For any $g \in \mathcal{G}$ and $\delta > 0$, it is known that the following inequality holds with probability at least $1 - \delta$~\citep{Mohri:2012:Foundations}:
\begin{align}\label{eq:rademacherbound}
    \mathbb{E}_z[g(z)] \leq \frac{1}{m} \sum_{i=1}^{m} g(z_i) + 2 \cdot \mathcal{R}_m(\mathcal{G}) + \sqrt{\frac{\log \frac{1}{\delta}}{2 \cdot m}}. 
\end{align}

Using this result, we show that the estimation error of the validity can be bounded by the Rademacher complexity of the family $\mathcal{H}$ in the following lemma. 
\begin{lemma}\label{lemm:rademacher}
    For a fixed classifier $h$ and action set $\calA$, let $\mathcal{H} \coloneqq \set{h_{\ba} \colon \mathcal{X} \to \mathcal{Y} \mid h_{\ba}(\bx) \coloneqq h(\bx + \ba), \ba \in \calA}$ be the family of functions that return the prediction value of $h$ after implementing an action $\ba$. 
    Then, for any $\ba \in \calA$ and $\delta > 0$, the following inequality holds with probability at least $1 - \delta$:
    \begin{align*}
        \hat{V}_{\tx}(\ba \mid S) - V_{\tx}(\ba) \leq \mathcal{R}_N(\mathcal{H}) + \sqrt{\frac{\ln \frac{1}{\delta}}{2 \cdot N}},
    \end{align*}
    where $\mathcal{R}_N(\mathcal{H}) = \mathbb{E}_{S}[\hat{\mathcal{R}}_S(\mathcal{H})]$ is the Rademacher complexity of $\mathcal{H}$ and $\hat{\mathcal{R}}_S(\mathcal{H})$ is the empirical Rademacher complexity of $\mathcal{H}$ over $N$ imputation candidates $S = \set{\hx_1, \dots \hx_N}$. 
\end{lemma}
\begin{proof}
    For notational convenience, let $i(\ba; \bx) \coloneqq \ind{h(\bx + \ba) \not= +1}$ be the \emph{invalidity} of an action $\ba$ for an instance $\bx$. 
    For an incomplete instance $\tx$ and its imputation candidates $S = \set{\hx_1, \dots \hx_N}$, we define the empirical invalidity and expected invalidity by $\hat{I}_{\tx}(\ba \mid S) \coloneqq \frac{1}{N} \sum_{n=1}^{N} i(\ba; \hx_n)$ and $I(\ba) \coloneqq \mathbb{E}_{\hx \sim \mathcal{D}_{\tx}}[i(\ba; \hx)]$, respectively. 
    Note that $\hat{I}_{\tx}(\ba \mid S) = 1 - \hat{V}_{\tx}(\ba \mid S)$ and $I_{\tx}(\ba) = 1 - V_{\tx}(\ba)$ hold from the definition. 

    Let $\mathcal{I} \coloneqq \set{i_{\ba} \colon \mathcal{X} \to \set{0, 1} \mid i_{\ba}(\bx) \coloneqq i(\ba; \bx), \ba \in \calA}$ be the family of the invalidity functions corresponding to the action set $\calA$. 
    By applying \eqref{eq:rademacherbound}, we have that the following inequality holds with probability at least $1 - \delta$ for any $i \in \mathcal{I}$ and $\delta > 0$:
    \begin{align*}
        \mathbb{E}_{\hx \sim \mathcal{D}_{\tx}}[i(\hx)] \leq \frac{1}{N} \sum_{n=1}^{N} i(\hx_n) + 2 \cdot \mathcal{R}_N(\mathcal{I}) + \sqrt{\frac{\log \frac{1}{\delta}}{2 \cdot N}}. 
    \end{align*}
    Since each $i \in \mathcal{I}$ corresponds to an action $\ba \in \calA$ and $i(\hx) = i(\ba; \hx)$ from the definition, we obtain
    \begin{align*}
        &\forall i \in \mathcal{I}: \mathbb{E}_{\hx \sim \mathcal{D}_{\tx}}[i(\hx)] \leq \frac{1}{N} \sum_{n=1}^{N} i(\hx_n) + 2 \cdot \mathcal{R}_N(\mathcal{I}) + \sqrt{\frac{\log \frac{1}{\delta}}{2 \cdot N}} \\
        &\iff \forall \ba \in \calA: \mathbb{E}_{\hx \sim \mathcal{D}_{\tx}}[i(\ba; \hx)] \leq \frac{1}{N} \sum_{n=1}^{N} i(\ba; \hx_n) + 2 \cdot \mathcal{R}_N(\mathcal{I}) + \sqrt{\frac{\log \frac{1}{\delta}}{2 \cdot N}} \\
        &\iff \forall \ba \in \calA: I_{\tx}(\ba) \leq \hat{I}_{\tx}(\ba \mid S) + 2 \cdot \mathcal{R}_N(\mathcal{I}) + \sqrt{\frac{\log \frac{1}{\delta}}{2 \cdot N}}. \\
        &\iff \forall \ba \in \calA: \hat{V}_{\tx}(\ba \mid S) \leq V_{\tx}(\ba) + 2 \cdot \mathcal{R}_N(\mathcal{I}) + \sqrt{\frac{\log \frac{1}{\delta}}{2 \cdot N}}. 
    \end{align*}

    From the definitions of the empirical Rademacher complexity, $\mathcal{I}$, and $\mathcal{H}$, we have
    \begin{align*}
        \hat{\mathcal{R}}_S(\mathcal{I}) 
        &= \mathbb{E}_{\sigma \in \set{\pm 1}^N} \left[ \sup_{\ba \in \calA} \frac{1}{N} \sum_{n=1}^{N} \sigma_{n} \cdot i(\ba; \hx_n) \right] \\
        &= \mathbb{E}_{\sigma \in \set{\pm 1}^N} \left[ \sup_{\ba \in \calA} \frac{1}{N} \sum_{n=1}^{N} \sigma_{n} \cdot \ind{h(\hx_n + \ba) = +1} \right] \\
        &= \mathbb{E}_{\sigma \in \set{\pm 1}^N} \left[ \sup_{\ba \in \calA} \frac{1}{N} \sum_{n=1}^{N} \sigma_{n} \cdot \frac{1 - h(\hx_n + \ba)}{2} \right] \\
        &= \frac{1}{2} \mathbb{E}_{\sigma \in \set{\pm 1}^N} \left[ \sup_{\ba \in \calA} \frac{1}{N} \sum_{n=1}^{N} - \sigma_{n} \cdot h(\hx_n + \ba) \right] \\
        &= \frac{1}{2} \cdot \mathbb{E}_{\sigma \in \set{\pm 1}^N} \left[ \sup_{\ba \in \calA} \frac{1}{N} \sum_{n=1}^{N} \sigma_{n} \cdot h(\hx_n + \ba) \right]
        = \frac{1}{2} \cdot \hat{\mathcal{R}}_S(\mathcal{H}).  
    \end{align*}
    From this result and the definition of the Rademacher complexity, we also have $\mathcal{R}_N(\mathcal{I}) = \frac{1}{2} \cdot \mathcal{R}_N(\mathcal{H})$ by taking expectations. 
    In summary, we obtain 
    \begin{align*}
        &\hat{V}_{\tx}(\ba \mid S) \leq V_{\tx}(\ba) + 2 \cdot \mathcal{R}_N(\mathcal{I}) + \sqrt{\frac{\log \frac{1}{\delta}}{2 \cdot N}} 
        = V_{\tx}(\ba) + \mathcal{R}_N(\mathcal{H}) + \sqrt{\frac{\log \frac{1}{\delta}}{2 \cdot N}} \\
        &\iff \hat{V}_{\tx}(\ba \mid S) - V_{\tx}(\ba) \leq \mathcal{R}_N(\mathcal{H}) + \sqrt{\frac{\log \frac{1}{\delta}}{2 \cdot N}}, 
    \end{align*}
    which concludes the proof. 
\end{proof}

Next, we consider to bound the Rademacher complexity $\mathcal{R}_N(\mathcal{H})$ by the \emph{growth function}~\citep{Mohri:2012:Foundations}. 
For a family $\mathcal{G}$ of functions $g \colon \mathcal{Z} \to \set{\pm 1}$ and $m \in \mathbb{N}$, the growth function for $\mathcal{G}$ is defined by $\Pi_{\mathcal{G}}(m) \coloneqq \max_{\set{z_1, \dots z_m} \subseteq \mathcal{Z}} |\set{(g(z_1), \dots, g(z_m)) \mid g \in \mathcal{G}}|$. 
It is known that the Rademacher complexity can be bounded by the growth function as follows~\citep{Mohri:2012:Foundations}:
\begin{align}\label{eq:growthbound}
    \mathcal{R}_m(\mathcal{G}) \leq \sqrt{\frac{2 \cdot \ln \Pi_{\mathcal{G}}(m)}{m}}. 
\end{align}
As a special case, we show the growth function of the family $\mathcal{H}$ for the case where its fixed classifier $h$ is a linear classifier in the following lemma. 
\begin{lemma}\label{lemm:linear}
    We consider the same setting with \cref{lemm:rademacher}. 
    Furthermore, we assume a linear classifier $h_{\bm{\beta}}(\bx) = \operatorname{sgn}(\bm{\beta}^\top \bx)$ with $\bm{\beta} = (\beta_1, \dots, \beta_D) \in \mathbb{R}^D$ and that $\calA$ is a convex set and contains at least one action $\ba$ such that $\ba \not= \bm{0}$. 
    Then, we have $\Pi_{\mathcal{H}}(N) = N + 1$. 
\end{lemma}
\begin{proof}
    By definition, the growth function of $\mathcal{H}$ with a fixed linear classifier $h_{\bm{\beta}}$ can be expressed as $\Pi_{\mathcal{H}}(N) = \max_{\set{\hx_1, \dots, \hx_N} \subseteq \mathcal{X}} |\set{(h_{\bm{\beta}}(\hx_1 + \ba), \dots , h_{\bm{\beta}}(\hx_N + \ba)) \mid \ba \in \calA}|$. 

    First, we consider to obtain an upper bound of the growth function $\Pi_{\mathcal{H}}(N)$. 
    Without loss of generality, we assume $\bm{\beta}^\top \hx_1 < \dots < \bm{\beta}^\top \hx_N$. 
    Then, for any $\ba \in \cal{A}$, we have 
    \begin{align*}
        \bm{\beta}^\top \hx_1 < \dots < \bm{\beta}^\top \hx_N
        &\implies \bm{\beta}^\top \hx_1 + \bm{\beta}^\top \ba  < \dots < \bm{\beta}^\top \hx_N + \bm{\beta}^\top \ba \\
        &\iff \bm{\beta}^\top (\hx_1 + \ba)  < \dots < \bm{\beta}^\top (\hx_N + \ba) \\
        &\implies h_{\bm{\beta}}(\hx_1 + \ba) \leq \dots \leq h_{\bm{\beta}}(\hx_N + \ba). 
    \end{align*}
    Since $h_{\bm{\beta}}(\hx + \ba)$ takes only $-1$ or $+1$, the total number of the possible patterns of $(h_{\bm{\beta}}(\hx_1 + \ba), \dots, h_{\bm{\beta}}(\hx_N + \ba))$ is at most $N + 1$, which implies $\Pi_{\mathcal{H}}(N) \leq N + 1$. 

    Next, we consider to obtain a lower bound of the growth function $\Pi_{\mathcal{H}}(N)$. 
    From the assumption, there exists an action $\ba \in \calA$ such that $\ba \not= \bm{0}$. 
    Then, let $S = \set{\hx_1, \dots, \hx_N} \subseteq \mathcal{X}$ be $N$ imputation candidates such that $\bm{\beta}^\top \hx_1 < \dots < \bm{\beta}^\top \hx_N < 0$ and $\bm{\beta}^\top (\hx_1 + \ba) \geq 0$. 
    Since $\calA$ is a convex set, for any $n \in \set{2, \dots, N}$, there exists an action $\ba' \in \calA$ such that $\ba' = \alpha \cdot \ba$ for some $\alpha \in (0, 1)$ and
    \begin{align*}
        \bm{\beta}^\top (\hx_1 + \ba')  < \dots < \bm{\beta}^\top (\hx_{n-1} + \ba') < 0 \leq \bm{\beta}^\top (\hx_{n} + \ba') < \dots < \bm{\beta}^\top (\hx_N + \ba'). 
    \end{align*}
    That is, for any $n \in \set{2, \dots, N}$, there exists an action $\ba' \in \calA$ such that $h_{\bm{\beta}}(\hx_{i} + \ba') = -1$ for $i \in \set{1, \dots, n-1}$ and $h_{\bm{\beta}}(\hx_{j} + \ba') = +1$ for $j \in \set{n, \dots, N}$. 
    In addition, there exists $\ba_0 \in \calA$ such that $h_{\bm{\beta}}(\hx_{i} + \ba_0) = -1$ for any $n \in [N]$ since $\bm{0} \in \calA$ by definition, and $h_{\bm{\beta}}(\hx_{i} + \ba) = +1$ holds for any $n \in [N]$ from the definition of $S$. 
    To wrap up the above facts, the total number of the possible patterns of $(h_{\bm{\beta}}(\hx_1 + \ba), \dots, h_{\bm{\beta}}(\hx_N + \ba))$ for $S$ is $N + 1$, which implies $\Pi_{\mathcal{H}}(N) \geq N + 1$. 

    From the above results, we have $\Pi_{\mathcal{H}}(N) = N + 1$. 
\end{proof}

Combining the above results, we give our proof of \cref{prop:vc} as follows. 
\begin{proof}[Proof of \cref{prop:vc}]
    From \cref{lemm:rademacher} and \eqref{eq:growthbound}, the following inequality holds with probability at least $1 - \delta$:
    \begin{align*}
        \hat{V}_{\tx}(\ba \mid S) - V_{\tx}(\ba) 
        \leq \mathcal{R}_N(\mathcal{H}) + \sqrt{\frac{\ln \frac{1}{\delta}}{2 \cdot N}} 
        \leq \sqrt{\frac{2 \cdot \ln \Pi_{\mathcal{H}}(N)}{N}} + \sqrt{\frac{\ln \frac{1}{\delta}}{2 \cdot N}} 
    \end{align*}
    which concludes the proof of the first statement of \cref{prop:vc}. 
    Furthermore, for the case where the fixed classifier $h$ is a linear classifier, we have $\Pi_{\mathcal{H}}(N) = N + 1$ from \cref{lemm:linear}, which concludes the proof of the second statement of \cref{prop:vc}. 
\end{proof}

\section{MILO formulations for neural networks and tree ensembles}\label{sec:appendix:formulation}
In this section, we propose MILO formulations of the problem~\eqref{eq:armin} for deep ReLU networks and tree ensembles. 
As with \cref{sec:optimization:milo}, we consider a linear cost function $c(\ba) = \sum_{d=1}^{D} c_{d}(a_{d})$, and assume that each coordinate $\calA_d$ of a given feasible action set $\calA = \calA_1 \times \dots \times \calA_D$ is finite and discretized; that is, we assume $\calA_d = \set{a_{d,1}, \dots, a_{d, J_d}}$, where $J_d = |\calA_d|$. 
To express an action $\ba \in \calA$, we introduce binary variables $\pi_{d,j} \in \set{0,1}$ for $d \in [D]$ and $j \in [ J_d ]$, which indicate that the action $a_{d,j} \in A_d$ is selected ($\pi_{d,j}=1$) or not ($\pi_{d,j}=0$). 
We also introduce auxiliary variables $\nu_n \in \set{0,1}$ for $n \in [N]$ such that $\nu_n = v(\ba; \hx_n)$. 

\subsection{Deep ReLU Networks}
For simplicity, we focus on a two-layer ReLU network $h(\bx) = \operatorname{sgn}\left( \sum_{t=1}^{T} \theta_t \cdot \max \set{0, \bm{\beta}_{t}^{\top} \bx } \right)$,
where $\bm{\beta}_t = (\beta_{t,1}, \dots, \beta_{t,D}) \in \mathbb{R}^{D}$ is a coefficient vector of the $t$-th neuron, $\theta_t \in \mathbb{R}$ is a weight value of the $t$-th neuron, and $T \in \mathbb{N}$ is the total number of neurons in the middle layer. 
Then, the problem~\eqref{eq:armin} with the two-layer ReLU network $h$ can be formulated as the following MILO problem~\citep{Serra:ICML2018,Kanamori:AAAI2021}:
\begin{align}\label{eq:mlp}
\renewcommand{\arraystretch}{1.5}
    \begin{array}{cl}
        \text{minimize}   &\displaystyle \sum_{d=1}^{D} \sum_{j=1}^{J_d} c_{d}(a_{d,j}) \cdot \pi_{d,j} \\
        \text{subject to} &\displaystyle \sum_{n=1}^{N} \nu_n \geq N \cdot \rho, \\
                          &\displaystyle \sum_{j=1}^{J_d} \pi_{d,j} = 1, \forall d \in [D], \\
                          &\displaystyle \sum_{t=1}^{T} \theta_t \cdot \xi_{n, t} \geq M_n \cdot (1 - \nu_n), \forall n \in [N], \\
                          &\displaystyle \xi_{n, t} \leq H_{n, t} \cdot \zeta_{n, t}, \forall t \in [T], \forall n \in [N], \\
                          &\displaystyle \bar{\xi}_{n, t} \leq - \bar{H}_{n, t} \cdot (1-\zeta_{n, t}), \forall t \in [T], \forall n \in [N], \\
                          &\displaystyle \xi_{n, t} - \bar{\xi}_{n, t} = \sum_{d=1}^{D} \beta_{t,d} \cdot \sum_{j=1}^{J_d} a_{d,j} \cdot \pi_{d,j} + F_{n, t},  \forall t \in [T], \forall n \in [N], \\
                          &\displaystyle \pi_{d,j} \in \set{0,1}, \forall d \in [D], \forall j \in [J_d], \\
                          &\displaystyle \xi_{n, t}, \bar{\xi}_{n, t} \geq 0, \zeta_{n, t} \in \set{0,1}, \forall t \in [T], \forall n \in [N], \\
                          &\displaystyle \nu_{n} \in \set{0,1}, \forall n \in [N], 
    \end{array}
\renewcommand{\arraystretch}{1.0}
\end{align}
where $M_n$, $H_{n, t}$, $\bar{H}_{n, t}$, and $F_{n, t}$ are constants such that $M_n \leq \min_{\ba \in \calA} \sum_{t=1}^{T} \theta_t \cdot \max \set{0, \bm{\beta}_{t}^{\top} \hx_n }$, $H_{n, t} \geq \max_{\ba \in \calA} \bm{\beta}_t^\top(\hx_n + \ba)$, $\bar{H}_{n, t} \leq \min_{\ba \in \calA} \bm{\beta}_t^\top(\hx_n + \ba)$, and $F_{n, t} = \bm{\beta}_t^\top \hx_n$, respectively. 
These values can be computed when $h$, $\bx$, and $\calA$ are given. 
Note that our formulation can be extended to general multilayer ReLU networks~\citep{Serra:ICML2018}. 

\subsection{Tree Ensembles}
Let $h$ be a tree ensemble $h(\bx) = \operatorname{sgn}\left(\sum_{t=1}^{T} \theta_t \cdot f_t(\bx) \right)$, where $f_t \colon \mathcal{X} \to \mathbb{R}$ is a decision tree, $\theta_t \in \mathbb{R}$ is a weight value of the $t$-th decision tree $f_t$, and $T \in \mathbb{N}$ is the total number of decision trees. 
Each decision tree $f_t$ can be expressed as $f_t(\bx) = \sum_{l=1}^{L_t} \hat{y}_{t,l} \cdot \ind{\bx \in r_{t,l}}$, where $L_t \in \mathbb{N}$ is the total number of leaves in $f_t$, and $\hat{y}_{t,l} \in \mathbb{R}$ and $r_{t,l} = r^{(1)}_{t,l} \times \dots \times r^{(D)}_{t,l} \subseteq \mathcal{X}$ are the predictive label and the region corresponding to a leaf $l \in [L_t]$, respectively. 
Then, the problem~\eqref{eq:armin} with the tree ensemble $h$ can be formulated as the following MILO problem~\citep{Cui:KDD2015,Kanamori:IJCAI2020}:
\begin{align}\label{eq:tree}
\renewcommand{\arraystretch}{1.5}
    \begin{array}{cl}
        \text{minimize}   &\displaystyle \sum_{d=1}^{D} \sum_{j=1}^{J_d} c_{d}(a_{d,j}) \cdot \pi_{d,j} \\
        \text{subject to} &\displaystyle \sum_{n=1}^{N} \nu_n \geq N \cdot \rho, \\
                          &\displaystyle \sum_{j=1}^{J_d} \pi_{d,j} = 1, \forall d \in [D], \\
                          &\displaystyle \sum_{t=1}^{T} \theta_t \cdot \sum_{l=1}^{L_t} \hat{y}_{t,l} \cdot \phi_{n, t,l} \geq M_n \cdot (1 - \nu_n), \forall n \in [N], \\
                          &\displaystyle \sum_{l=1}^{L_t} \phi_{n, t,l} = 1, \forall t \in [T], \forall n \in [N], \\
                          &\displaystyle D \cdot \phi_{n, t,l} \leq \sum_{d=1}^{D} \sum_{j \in J^{(d)}_{n,t,l}} \pi_{n,d,j}, \forall t \in [T], \forall l \in [L_t], \forall n \in [N], \\
                          &\displaystyle \pi_{d,j} \in \set{0,1}, \forall d \in [D], \forall j \in [J_d], \\
                          &\displaystyle \phi_{n, t,l} \in \set{0,1}, \forall t \in [T], \forall l \in [L_t], \forall n \in [N], \\
                          &\displaystyle \nu_{n} \in \set{0,1}, \forall n \in [N], 
    \end{array}
\renewcommand{\arraystretch}{1.0}
\end{align}
where $M_n \leq \min_{\ba \in \calA} \sum_{t=1}^{T} \theta_t \cdot f_t(\hx_n)$ and $J^{(d)}_{n, t, l} = \set{j \in [J_d] \mid \hat{x}_{n, d} + a_{d,j} \in r^{(d)}_{t,l}}$, which can be computed when $h$, $\hx_n$, and $\calA$ are given.

\section{Implementation details}\label{sec:appendix:implementation}
\subsection{ImputationAR}
As a baseline that works with missing values, we implemented a naive method that combines the existing AR methods with an imputation method. 
Let $i \colon \tilde{\mathcal{X}} \to \mathcal{X}$ be an imputation method that replaces the missing values of a given incomplete instance $\tx \in \tilde{\mathcal{X}}$ with some plausible values and returns an imputed instance $\hx \in \mathcal{X}$. 
Our ImputationAR with an imputation method $i$ consists of the following two steps:
\begin{enumerate}
    \item For a given instance $\tx \in \tilde{\mathcal{X}}$ with missing values, we obtain its imputed instance $\hx = i(\tx)$ by applying the imputation method $i$. 
    \item We optimize an action for the imputed instance $\hx$ instead of $\tx$ by calculating an optimal action $\ba^\ast$ for $\hx$ using the existing method for solving the problem~\eqref{eq:ce}. 
\end{enumerate}
As an imputation method $i$, we employ three major methods: mean imputation~\citep{Little:2019:Missing}, $k$-NN imputation~\citep{Troyanskaya:Bioinformatics2001}, and MICE~\citep{Buuren:JSS2011}. 
We implemented each method using scikit-learn\footnote{\url{https://scikit-learn.org/stable/modules/classes.html\#module-sklearn.impute}} with its default parameters. 
Note that we implemented MICE with \texttt{IterativeImputer} that adapts MICE to be able to impute test instances, as with previous studies~\citep{LeMorvan:NIPS2021}.

\subsection{RobustAR}
As another baseline method, we extend the existing robust AR methods~\citep{Dominguez-olmedo:ICML2022,Dutta:ICML2022,Pawelczyk:arxiv2022} to our setting. 
Most existing studies on robust AR aim to optimize an action that is valid even if a given instance is slightly perturbed. 
For a given instance $\bx \in \mathcal{X}$ without missing values, their task can be formulated as the following optimization problem: 
\begin{align*}
\renewcommand{\arraystretch}{1.5}
    \begin{array}{cll}
        \displaystyle \mathop{\text{\upshape minimize}}_{\ba \in \calA} & c(\ba) & \\
        \text{\upshape subject to} & h(\bx' + \ba) = +1, &\forall \bx' \in B_\varepsilon(\bx),
        \end{array}
\renewcommand{\arraystretch}{1.0}
\end{align*}
where $B_\varepsilon(\bx) = \set{ \bx' \in \mathcal{X} \mid \| \bx - \bx' \| \leq \varepsilon }$ is the $\varepsilon$-ball of $\bx$, i.e., set of perturbed instances around $\bx$, for some $\varepsilon > 0$. 
This formulation is motivated by the observation that actions $\ba$ are often affected by small perturbations to inputs $\bx$, which is similar to our observation that actions are affected by imputation.

To extend the existing robust AR methods to handle our setting with missing values, we take the following two steps: 
(1)~obtain the imputed instance $\hx = i(\tx)$ by applying an imputation method $i$ (e.g., MICE~\citep{Buuren:JSS2011}) for $\tx$, and  
(2)~replace $B_\varepsilon(\bx)$ with a set $S = \set{\hx_1, \dots, \hx_N}$ of $N$ imputation candidates randomly sampled from the distribution $\mathcal{D}_{\tx}$. 
Overall, our RobustAR method optimizes an action for $\tx$ by solving the following problem:
\begin{align}\label{prob:appendix:rce}
\renewcommand{\arraystretch}{1.5}
    \begin{array}{cll}
        \displaystyle \mathop{\text{\upshape minimize}}_{a \in \calA(\hx)} & c(a \mid \hx) & \\
        \text{\upshape subject to} & h(\hx+a)=+1, &  \\
        & h(\hx'+a)=+1, &\forall \hx' \in S.
        \end{array}
\renewcommand{\arraystretch}{1.0}
\end{align}
By solving the problem \eqref{prob:appendix:rce}, we can obtain an action that is valid for any $\hx' \in (\set{\hx} \cup S)$, which indicates that the obtained action may also be valid for the original instance $\bx$. 

Fortunately, the problem \eqref{prob:appendix:rce} can be solved by extending the existing MILO-based AR methods to include linear constraints that express the additional constraints $h(\hx'+a)=+1$ for $\hx' \in S$. 
However, such additional constraints increase the total number of constraints in the MILO problem, which makes solving the problem \eqref{prob:appendix:rce} challenging for a large $N$ and a complex classifier $h$. 
To address this computational issue, we modify the optimization algorithm of the existing robust AR method proposed by~\citep{Dominguez-olmedo:ICML2022}. 
Our modified algorithm for the problem \eqref{prob:appendix:rce} consists of the following steps: 
\begin{enumerate}
    \item Optimize an action $\hat{\ba}^\ast = \ba^\ast(\hx)$ for $\hx$ by the MILO-based method without any additional constraint. 
    \item Find $\hx^\ast = \arg\max_{\hx' \in S} l_\mathrm{logistic}(f(\hx'+\hat{\ba}^\ast), +1)$, where $l_\mathrm{logistic}$ is the logistic loss and $f$ is the decision function of $h$ such that $h(\bx) = \operatorname{sgn}(f(\bx))$. 
    \item If $h(\hx^\ast + \hat{\ba}^\ast) = +1$, then return $\hat{\ba}^\ast$
    \item Update an action $\hat{\ba}^\ast$ by adding the constraint $h(\hx^\ast + \ba) = +1$ to the MILO formulation and solving the MILO problem, and go to Step 2. 
\end{enumerate}
The above algorithm avoids increasing constraints by sequentially adding the constraints one by one to the MILO formulation. 
Because the action $\hat{\ba}^\ast$ obtained by the algorithm satisfies $h(\hx'+\hat{\ba}^\ast)=+1$ for all $\hx' \in (\set{\hx} \cup S)$ and minimizes its cost, we can guarantee that $\hat{\ba}^\ast$ is an optimal solution to the problem~\eqref{prob:appendix:rce}. 
In our experiments, we observed that the above algorithm often stopped after about $10$ iterations even for $N \geq 100$, and it was faster than adding all the constraints to the MILO formulation.

\subsection{ARMIN}\label{sec:appendix:implementation:armin}
\myparagraph{Heuristic solution approach.}
As shown in \eqref{eq:linear}, \eqref{eq:mlp} and \eqref{eq:tree}, we can formulate \eqref{eq:armin} as a MILO problem by extending the existing MILO-based AR methods. 
However, solving the formulated problem is challenging for a large $N$ and a complex classifier $h$ (e.g., a random forest classifier with over $200$ decision trees). 
This is because the total number of constraints in the formulated problem increases depending on $N$ and $h$. 
Indeed, in our preliminary experiments, we often failed to obtain low-cost actions within a given time limit. 
To alleviate this issue, we propose a heuristic approach that divides the problem into a set of a few small problems by subsampling the imputation candidates. 
This procedure consists of the following steps:
\begin{enumerate}
    \item Obtain a random subsample $S'$ with $N' \ll N$ imputation candidates from $S$. 
    \item Solve the formulated MILO problem with $S'$ and obtain a solution $\ba$. 
    \item Repeat the above steps $P$ times and return a solution $\ba$ with the minimum cost $c(\ba)$ among the solutions such that $V_{\tx}(\ba; S) \geq \rho$. 
\end{enumerate}
In our experiments, we set $N' = 10$ and $P = 10$, and a 60-second time limit was imposed for obtaining each solution. 
We confirmed that this approach could improve the quality of actions without increasing the overall computational time compared to solving the full problem formulated using $S$. 

\begin{wrapfigure}[9]{R}{0.5\textwidth}
    \vspace{-8mm}
    \begin{minipage}{0.5\textwidth}
        \begin{algorithm}[H]
            \small
            \caption{Path algorithm for \eqref{eq:armin}. }
            \begin{algorithmic}[1]
                \State $t \leftarrow 0$; $\rho_0 \leftarrow \frac{1}{N}$; 
                \While{$\rho_t < 1$}
                    \State $\ba_t \leftarrow \alg{MILO}(\rho_t)$; \Comment{Solve \eqref{eq:armin} with $\rho_t$}
                    \State $\rho_{t+1} \leftarrow \hat{V}_{\tx}(\ba_t \mid S) + \frac{1}{N}$; \Comment{Update to $\rho_{t+1}$}
                    \State $t \leftarrow t + 1$; 
                \EndWhile
                \State \textbf{return} $\set{(\rho_0, \ba_0), \dots, (\rho_t, \ba_t)}$; 
            \end{algorithmic}
            \label{algo:path}
        \end{algorithm}
    \end{minipage}
\end{wrapfigure}
\myparagraph{Path algorithm.}
Our formulation has a hyperparameter $\rho$ that is the threshold for the empirical validity $\hat{V}_{\tx}(\ba \mid S)$. 
We expect that the larger $\rho$ becomes, the higher the probability that an obtained action is valid for $\bx$. 
In contrast, the cost of an obtained action with a large $\rho$ tends to become high~\citep{Pawelczyk:arxiv2022}. 
To efficiently analyze such a trade-off, we propose a path algorithm in \cref{algo:path}. 
It is based on the fact that the empirical validity $\hat{V}_{\tx}$ takes only $\frac{1}{N}, \frac{2}{N} \dots, \frac{N}{N}$. 
Thus, we can determine the next value of $\rho$ to be searched by adding $\frac{1}{N}$ to the empirical validity of a current solution, and this update is repeated at most $N$ times. 

\myparagraph{On the relationship to robust AR.}
When we set $\rho = 1$ and regard imputation candidates $S$ as the uncertainty set, the problem of \eqref{eq:armin} is roughly equivalent to \eqref{prob:appendix:rce}, i.e., the existing robust AR problem~\citep{Dominguez-olmedo:ICML2022}. 
One of the differences between the robust AR and our ARMIN is whether we can tune the confidence parameter $\rho$. 
It is trivial that the larger the confidence $\rho$ is, the higher the optimal cost of \eqref{eq:armin} becomes. 
Indeed, the existing robust AR methods are known to tend to provide high-cost actions~\citep{Pawelczyk:arxiv2022}. 
Furthermore, as shown in \cref{fig:exp:path,fig:appendix:exp:pathmar,fig:appendix:exp:pathmnar}, we empirically confirmed that the costs increased sharply when the confidences $\rho$ exceed about $0.9$. 
This may be related to our experimental results of the baseline comparison that the average cost of RobustAR was higher than the others.

\section{Additional experimental results}\label{sec:appendix:experiments}

\subsection{On the models that can handle missing values without imputation}\label{sec:appendix:experiments:gbdt}
In the main paper, we discussed the risk that the imputation of missing values affects the resulting recourse actions. 
In practice, there exist some models that can handle missing values without imputation. 
For example, XGBoost~\citep{Chen:KDD2016} can handle input instances with missing values by learning a default branching direction for incomplete instances in each internal node of decision trees. 
Even for such models, however, we argue that the presence of missing values can affect the resulting actions for the following reason: if the branching direction of an input instance in each decision tree is changed by its missing values, then the actions required for altering the prediction result of the instance can also be changed. 

To demonstrate the above risk, we conducted additional experiments with the XGBoost classifier. 
As with the experiments in the main paper, we trained XGBoost classifiers on each dataset and obtained an optimal action for each test instance with missing values by the existing AR method based on MILO~\citep{Cui:KDD2015,Kanamori:IJCAI2020}. 
We varied the total number of missing values in test instances and measured the valid ratio, average cost, and average sign agreement score of the obtained actions. 

\cref{fig:appendix:exp:gbdt} present the experimental results. 
From \cref{fig:appendix:exp:gbdt}, we can see that the valid ratio and average sign agreement score decreased significantly as the total number of missing values increased for all the datasets. 
\cref{tab:appendix:example} shows examples of actions extracted from an XGBoost classifier on the FICO dataset. 
\cref{tab:appendix:example:original,tab:appendix:example:incomplete} respectively present the obtained actions for an original instance and its incomplete instance where the feature ``PercentInstallTrades" is missing. 
From \cref{tab:appendix:example}, we observed that the obtained actions were quantitatively and qualitatively different from each other due to missing just one single feature. 
From these results, we empirically confirmed that the presence of missing values certainly affects the resulting action, regardless of whether a model can handle missing values without imputation.

\subsection{Complete results of MCAR situation}
We show complete experimental results under the MCAR mechanism. 
\cref{fig:appendix:exp:mcar1,fig:appendix:exp:mcar2,fig:appendix:exp:mcar3} present the complete experimental results of our baseline comparison under the MCAR situation with $D_{\ast} \in \set{1, 2, 3}$. 
We also show the results of the average computational time in \cref{tab:appendix:exp:time1,tab:appendix:exp:time2,tab:appendix:exp:time3} and the average sign agreement score in \cref{fig:appendix:exp:sas1,fig:appendix:exp:sas2,fig:appendix:exp:sas3}.

\subsection{Additional results of trade-off analysis}
\cref{fig:appendix:exp:pathmar,fig:appendix:exp:pathmnar} shows the additional results of our path analysis under the MAR and MNAR situations, respectively. 
We also show the results of the sensitivity analyses of the confidence parameter $\rho$ in \cref{fig:appendix:exp:confidence}, where we considered the MCAR situation and set $D_{\ast} = 2$.

\subsection{Sensitivity analyses of sampling size}
To examine the sensitivity with respect to the sampling size $N$ of our ARMIN, we conducted its sensitivity analyses by varying the values of $N$. 
We used the LR classifier on each dataset under the MCAR mechanism and set $D_{\ast} = 2$. 
We varied the value of $N$ in $\set{100,200,300,400,500}$, and measured the valid ratio, average cost, average sign agreement score, and average computational time. 
\cref{fig:appendix:exp:sampling} presents the experimental results. 
We can see that the sampling size $N$ did not significantly affect the quality of solutions except for the average computational time. 
These results indicate that we can obtain sufficiently good actions by our ARMIN even with small $N$, which encourages computational efficiency.

\section{Additional comments on existing assets}\label{sec:appendix:assets}
Gurobi 10.0.3\footnote{\url{https://www.gurobi.com/}} is a commercial solver for mixed-integer optimization provided by Gurobi Optimization, LLC.  
Scikit-learn 1.0.2 \footnote{\url{https://scikit-learn.org/stable/}} is publicly available under the BSD-3-Clause license. 

All datasets used in \cref{sec:experiments} are publicly available and do not contain any identifiable information or offensive content. 
As they are accompanied by appropriate citations in the main body, see the corresponding references for more details.

\section{Discussion on potential societal impacts}\label{sec:appendix:impacts}
Our proposed method, named algorithmic recourse with multiple imputation (ARMIN), is a new framework of algorithmic recourse that works even in the presence of missing values. 
Our method can provide actions for altering the given prediction results into the desired one, which is also recognized as ``recourse~\citep{Ustun:FAT*2019,Karimi:FAccT2021}," even if some features of the input instances are missing. 
Since recent studies have pointed out the risk that an adversary can leverage recourse actions to infer private information about the data held by decision-makers~\citep{Pawelczyk:AISTATS2023}, it is important to allow users to avoid disclosing some feature values corresponding to their private information (e.g., income). 
Our ARMIN enables users to obtain valid and low-cost actions without acquiring the true values of the missing features that the users do not wish to disclose.

\clearpage
\begin{figure}[p]
    \centering
    \includegraphics[width=\textwidth]{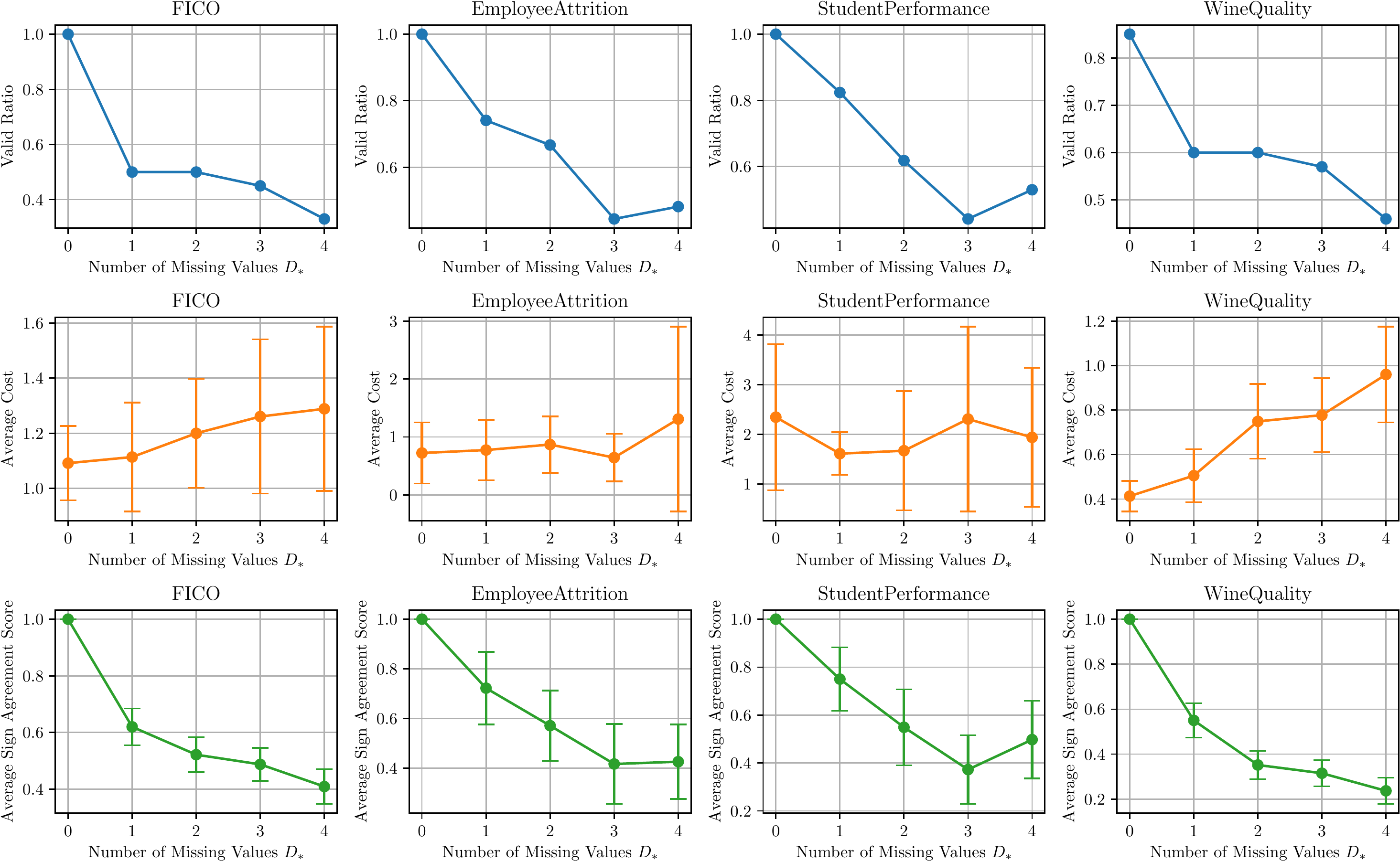}
    \caption{
        Experimental results of the XGBoost classifier with missing values $D_{\ast} \in \set{0, 1, 2, 3, 4}$ under the MCAR situation. 
    }
    \label{fig:appendix:exp:gbdt}
\end{figure}

\begin{table}[p]
    \centering
    \caption{
        Examples of actions extracted from an XGBoost classifier on the FICO dataset. 
        Here, we generate the incomplete instance $\tx$ by dropping the value of the feature ``PercentInstallTrades" in its original instance $\bx$. 
    }
    \subfigure[Optimal action for original instance $\bx$]{
        \adjustbox{valign=b}{
            \begin{tabular}{lccc}
            \toprule
                \textbf{Feature} & \textbf{Value}  & \textbf{Valid}  & \textbf{Cost} \\
            \midrule
                ExternalRiskEstimate & $+20$ & True & $1.35$ \\
            \bottomrule
            \end{tabular}
            \label{tab:appendix:example:original}
        }
    }
    \hfill
    \subfigure[Optimal action for incomplete instance $\tx$]{
        \adjustbox{valign=b}{
            \begin{tabular}{lccc}
            \toprule
                \textbf{Feature} & \textbf{Value}  & \textbf{Valid}  & \textbf{Cost} \\
            \midrule
                ExternalRiskEstimate & $-6$ & \multirow{4}{*}{False} & \multirow{4}{*}{$1.14$} \\
                MaxDelq2PublicRecLast12M & $-2$ & & \\
                NetFractionRevolvingBurden & $-19$ & & \\
                PercentTradesWBalance & $-7$ & & \\
            \bottomrule
            \end{tabular}
            \label{tab:appendix:example:incomplete}
        }
    }
    \label{tab:appendix:example}
\end{table}

\clearpage
\begin{figure}[p]
    \centering
    \includegraphics[width=\textwidth]{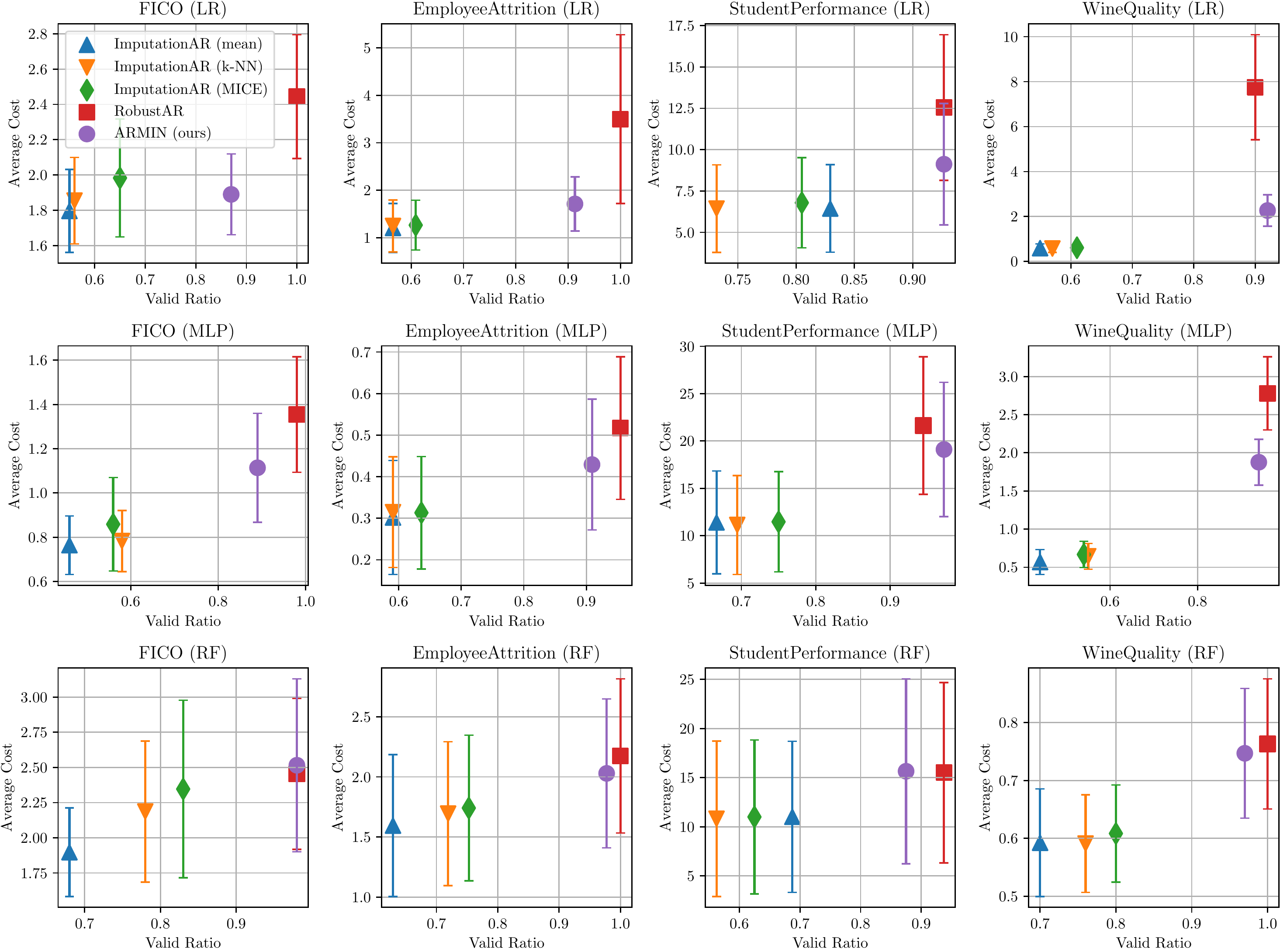}
    \caption{
        Experimental results of our comparison under the MCAR situation, where $D_{\ast} = 1$. 
    }
    \label{fig:appendix:exp:mcar1}
\end{figure}
\begin{figure}[p]
    \centering
    \includegraphics[width=\textwidth]{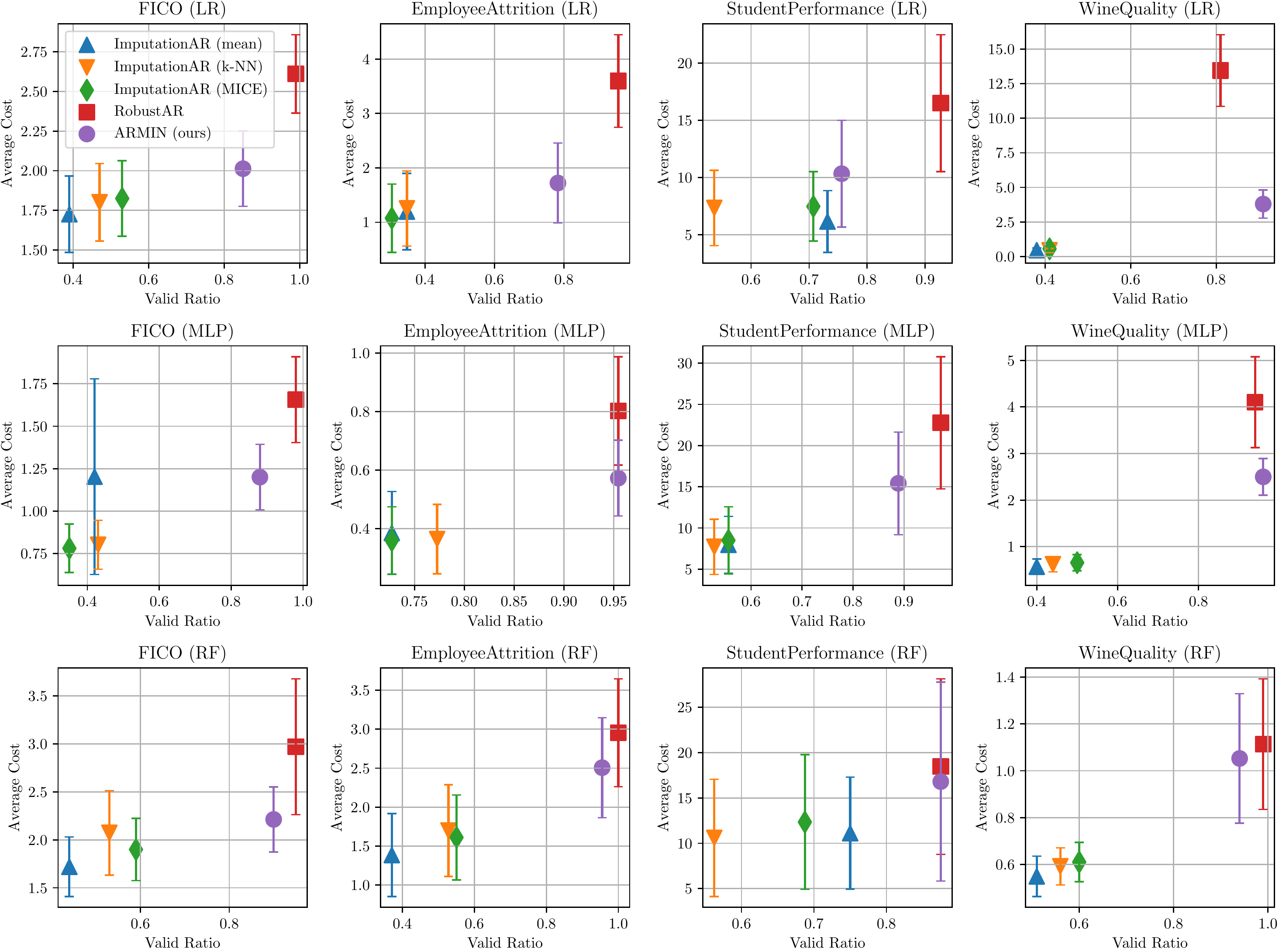}
    \caption{
        Experimental results of our comparison under the MCAR situation, where $D_{\ast} = 2$. 
    }
    \label{fig:appendix:exp:mcar2}
\end{figure}
\begin{figure}[p]
    \centering
    \includegraphics[width=\textwidth]{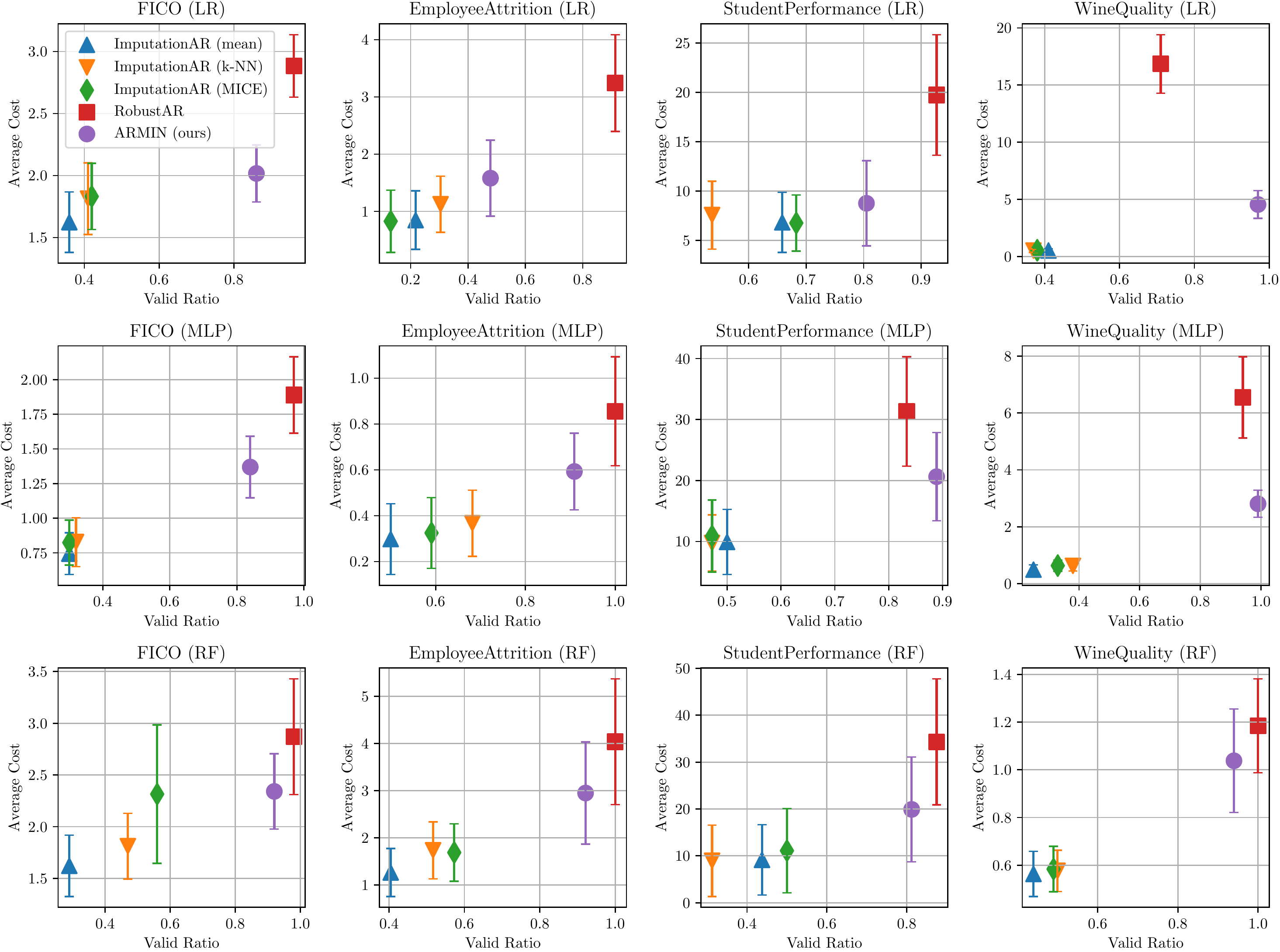}
    \caption{
        Experimental results of our comparison under the MCAR situation, where $D_{\ast} = 3$. 
    }
    \label{fig:appendix:exp:mcar3}
\end{figure}

\begin{table}[p]
    \centering
    \scriptsize
    \caption{Experimental results of the average computational time [s], where $D_{\ast} = 1$. }
    \begin{tabular}{ccccc}
    \toprule
    \textbf{Dataset} & \textbf{Method} & \textbf{LR} & \textbf{MLP} & \textbf{RF} \\
    \midrule
    \multirow{5}{*}{\textbf{ FICO }}
    & ImputationAR (mean) & $0.023113 \pm 0.008066$ & $0.29502 \pm 0.2939$ & $0.43635 \pm 0.471$ \\
    & ImputationAR (k-NN) & $0.023366 \pm 0.008898$ & $0.31684 \pm 0.3095$ & $0.44712 \pm 0.4693$ \\
    & ImputationAR (MICE) & $0.0234 \pm 0.007134$ & $0.32673 \pm 0.3261$ & $0.47649 \pm 0.5398$ \\
    & RobustAR & $0.13493 \pm 0.02382$ & $9.1928 \pm 37.16$ & $2.4749 \pm 6.716$ \\
    & ARMIN (ours) & $0.10886 \pm 0.03743$ & $96.47 \pm 120.5$ & $230.14 \pm 196.0$ \\
    \midrule
    \multirow{5}{*}{\textbf{ EmployeeAttrition }}
    & ImputationAR (mean) & $0.0046998 \pm 0.003787$ & $0.0091407 \pm 0.005912$ & $0.05703 \pm 0.0435$ \\
    & ImputationAR (k-NN) & $0.024596 \pm 0.01656$ & $0.028607 \pm 0.01142$ & $0.062058 \pm 0.04942$ \\
    & ImputationAR (MICE) & $0.0044343 \pm 0.003372$ & $0.010222 \pm 0.004549$ & $0.063377 \pm 0.05144$ \\
    & RobustAR & $0.045233 \pm 0.007768$ & $0.077863 \pm 0.04334$ & $0.71814 \pm 0.6171$ \\
    & ARMIN (ours) & $0.039063 \pm 0.009967$ & $1.3462 \pm 0.1676$ & $31.797 \pm 17.72$ \\
    \midrule
    \multirow{5}{*}{\textbf{ StudentPerformance }}
    & ImputationAR (mean) & $0.0016542 \pm 0.0008235$ & $0.0091867 \pm 0.005662$ & $0.0099284 \pm 0.003019$ \\
    & ImputationAR (k-NN) & $0.018182 \pm 0.009758$ & $0.034568 \pm 0.01035$ & $0.0098953 \pm 0.0039$ \\
    & ImputationAR (MICE) & $0.009501 \pm 0.005815$ & $0.0076689 \pm 0.005152$ & $0.010078 \pm 0.003142$ \\
    & RobustAR & $0.038643 \pm 0.04715$ & $0.35691 \pm 0.7044$ & $2.0854 \pm 3.715$ \\
    & ARMIN (ours) & $0.028376 \pm 0.006636$ & $5.9636 \pm 4.292$ & $13.485 \pm 1.389$ \\
    \midrule
    \multirow{5}{*}{\textbf{ WineQuality }}
    & ImputationAR (mean) & $0.0091039 \pm 0.005667$ & $0.033375 \pm 0.03401$ & $0.12362 \pm 0.1384$ \\
    & ImputationAR (k-NN) & $0.014566 \pm 0.006619$ & $0.037748 \pm 0.02964$ & $0.12467 \pm 0.1227$ \\
    & ImputationAR (MICE) & $0.010178 \pm 0.004696$ & $0.037779 \pm 0.04344$ & $0.13309 \pm 0.1464$ \\
    & RobustAR & $0.11023 \pm 0.07683$ & $0.46431 \pm 0.4061$ & $1.0654 \pm 2.24$ \\
    & ARMIN (ours) & $0.057723 \pm 0.01508$ & $13.603 \pm 5.981$ & $55.677 \pm 70.25$ \\
    \bottomrule
    \end{tabular}
    \label{tab:appendix:exp:time1}
\end{table}
\begin{table}[p]
    \centering
    \scriptsize
    \caption{Experimental results of the average computational time [s], where $D_{\ast} = 2$. }
    \begin{tabular}{ccccc}
    \toprule
    \textbf{Dataset} & \textbf{Method} & \textbf{LR} & \textbf{MLP} & \textbf{RF} \\
    \midrule
    \multirow{5}{*}{\textbf{ FICO }}
    & ImputationAR (mean) & $0.023113 \pm 0.008066$ & $0.29502 \pm 0.2939$ & $0.43635 \pm 0.471$ \\
    & ImputationAR (k-NN) & $0.023366 \pm 0.008898$ & $0.31684 \pm 0.3095$ & $0.44712 \pm 0.4693$ \\
    & ImputationAR (MICE) & $0.0234 \pm 0.007134$ & $0.32673 \pm 0.3261$ & $0.47649 \pm 0.5398$ \\
    & RobustAR & $0.13493 \pm 0.02382$ & $9.1928 \pm 37.16$ & $2.4749 \pm 6.716$ \\
    & ARMIN (ours) & $0.10886 \pm 0.03743$ & $96.47 \pm 120.5$ & $230.14 \pm 196.0$ \\
    \midrule
    \multirow{5}{*}{\textbf{ EmployeeAttrition }}
    & ImputationAR (mean) & $0.0046998 \pm 0.003787$ & $0.0091407 \pm 0.005912$ & $0.05703 \pm 0.0435$ \\
    & ImputationAR (k-NN) & $0.024596 \pm 0.01656$ & $0.028607 \pm 0.01142$ & $0.062058 \pm 0.04942$ \\
    & ImputationAR (MICE) & $0.0044343 \pm 0.003372$ & $0.010222 \pm 0.004549$ & $0.063377 \pm 0.05144$ \\
    & RobustAR & $0.045233 \pm 0.007768$ & $0.077863 \pm 0.04334$ & $0.71814 \pm 0.6171$ \\
    & ARMIN (ours) & $0.039063 \pm 0.009967$ & $1.3462 \pm 0.1676$ & $31.797 \pm 17.72$ \\
    \midrule
    \multirow{5}{*}{\textbf{ StudentPerformance }}
    & ImputationAR (mean) & $0.0016542 \pm 0.0008235$ & $0.0091867 \pm 0.005662$ & $0.0099284 \pm 0.003019$ \\
    & ImputationAR (k-NN) & $0.018182 \pm 0.009758$ & $0.034568 \pm 0.01035$ & $0.0098953 \pm 0.0039$ \\
    & ImputationAR (MICE) & $0.009501 \pm 0.005815$ & $0.0076689 \pm 0.005152$ & $0.010078 \pm 0.003142$ \\
    & RobustAR & $0.038643 \pm 0.04715$ & $0.35691 \pm 0.7044$ & $2.0854 \pm 3.715$ \\
    & ARMIN (ours) & $0.028376 \pm 0.006636$ & $5.9636 \pm 4.292$ & $13.485 \pm 1.389$ \\
    \midrule
    \multirow{5}{*}{\textbf{ WineQuality }}
    & ImputationAR (mean) & $0.0091039 \pm 0.005667$ & $0.033375 \pm 0.03401$ & $0.12362 \pm 0.1384$ \\
    & ImputationAR (k-NN) & $0.014566 \pm 0.006619$ & $0.037748 \pm 0.02964$ & $0.12467 \pm 0.1227$ \\
    & ImputationAR (MICE) & $0.010178 \pm 0.004696$ & $0.037779 \pm 0.04344$ & $0.13309 \pm 0.1464$ \\
    & RobustAR & $0.11023 \pm 0.07683$ & $0.46431 \pm 0.4061$ & $1.0654 \pm 2.24$ \\
    & ARMIN (ours) & $0.057723 \pm 0.01508$ & $13.603 \pm 5.981$ & $55.677 \pm 70.25$ \\
    \bottomrule
    \end{tabular}
    \label{tab:appendix:exp:time2}
\end{table}
\begin{table}[p]
    \centering
    \scriptsize
    \caption{Experimental results of the average computational time [s], where $D_{\ast} = 3$. }
    \begin{tabular}{ccccc}
    \toprule
    \textbf{Dataset} & \textbf{Method} & \textbf{LR} & \textbf{MLP} & \textbf{RF} \\
    \midrule
    \multirow{5}{*}{\textbf{ FICO }}
    & ImputationAR (mean) & $0.023113 \pm 0.008066$ & $0.29502 \pm 0.2939$ & $0.43635 \pm 0.471$ \\
    & ImputationAR (k-NN) & $0.023366 \pm 0.008898$ & $0.31684 \pm 0.3095$ & $0.44712 \pm 0.4693$ \\
    & ImputationAR (MICE) & $0.0234 \pm 0.007134$ & $0.32673 \pm 0.3261$ & $0.47649 \pm 0.5398$ \\
    & RobustAR & $0.13493 \pm 0.02382$ & $9.1928 \pm 37.16$ & $2.4749 \pm 6.716$ \\
    & ARMIN (ours) & $0.10886 \pm 0.03743$ & $96.47 \pm 120.5$ & $230.14 \pm 196.0$ \\
    \midrule
    \multirow{5}{*}{\textbf{ EmployeeAttrition }}
    & ImputationAR (mean) & $0.0046998 \pm 0.003787$ & $0.0091407 \pm 0.005912$ & $0.05703 \pm 0.0435$ \\
    & ImputationAR (k-NN) & $0.024596 \pm 0.01656$ & $0.028607 \pm 0.01142$ & $0.062058 \pm 0.04942$ \\
    & ImputationAR (MICE) & $0.0044343 \pm 0.003372$ & $0.010222 \pm 0.004549$ & $0.063377 \pm 0.05144$ \\
    & RobustAR & $0.045233 \pm 0.007768$ & $0.077863 \pm 0.04334$ & $0.71814 \pm 0.6171$ \\
    & ARMIN (ours) & $0.039063 \pm 0.009967$ & $1.3462 \pm 0.1676$ & $31.797 \pm 17.72$ \\
    \midrule
    \multirow{5}{*}{\textbf{ StudentPerformance }}
    & ImputationAR (mean) & $0.0016542 \pm 0.0008235$ & $0.0091867 \pm 0.005662$ & $0.0099284 \pm 0.003019$ \\
    & ImputationAR (k-NN) & $0.018182 \pm 0.009758$ & $0.034568 \pm 0.01035$ & $0.0098953 \pm 0.0039$ \\
    & ImputationAR (MICE) & $0.009501 \pm 0.005815$ & $0.0076689 \pm 0.005152$ & $0.010078 \pm 0.003142$ \\
    & RobustAR & $0.038643 \pm 0.04715$ & $0.35691 \pm 0.7044$ & $2.0854 \pm 3.715$ \\
    & ARMIN (ours) & $0.028376 \pm 0.006636$ & $5.9636 \pm 4.292$ & $13.485 \pm 1.389$ \\
    \midrule
    \multirow{5}{*}{\textbf{ WineQuality }}
    & ImputationAR (mean) & $0.0091039 \pm 0.005667$ & $0.033375 \pm 0.03401$ & $0.12362 \pm 0.1384$ \\
    & ImputationAR (k-NN) & $0.014566 \pm 0.006619$ & $0.037748 \pm 0.02964$ & $0.12467 \pm 0.1227$ \\
    & ImputationAR (MICE) & $0.010178 \pm 0.004696$ & $0.037779 \pm 0.04344$ & $0.13309 \pm 0.1464$ \\
    & RobustAR & $0.11023 \pm 0.07683$ & $0.46431 \pm 0.4061$ & $1.0654 \pm 2.24$ \\
    & ARMIN (ours) & $0.057723 \pm 0.01508$ & $13.603 \pm 5.981$ & $55.677 \pm 70.25$ \\
    \bottomrule
    \end{tabular}
    \label{tab:appendix:exp:time3}
\end{table}

\begin{figure}[p]
    \centering
    \includegraphics[width=\textwidth]{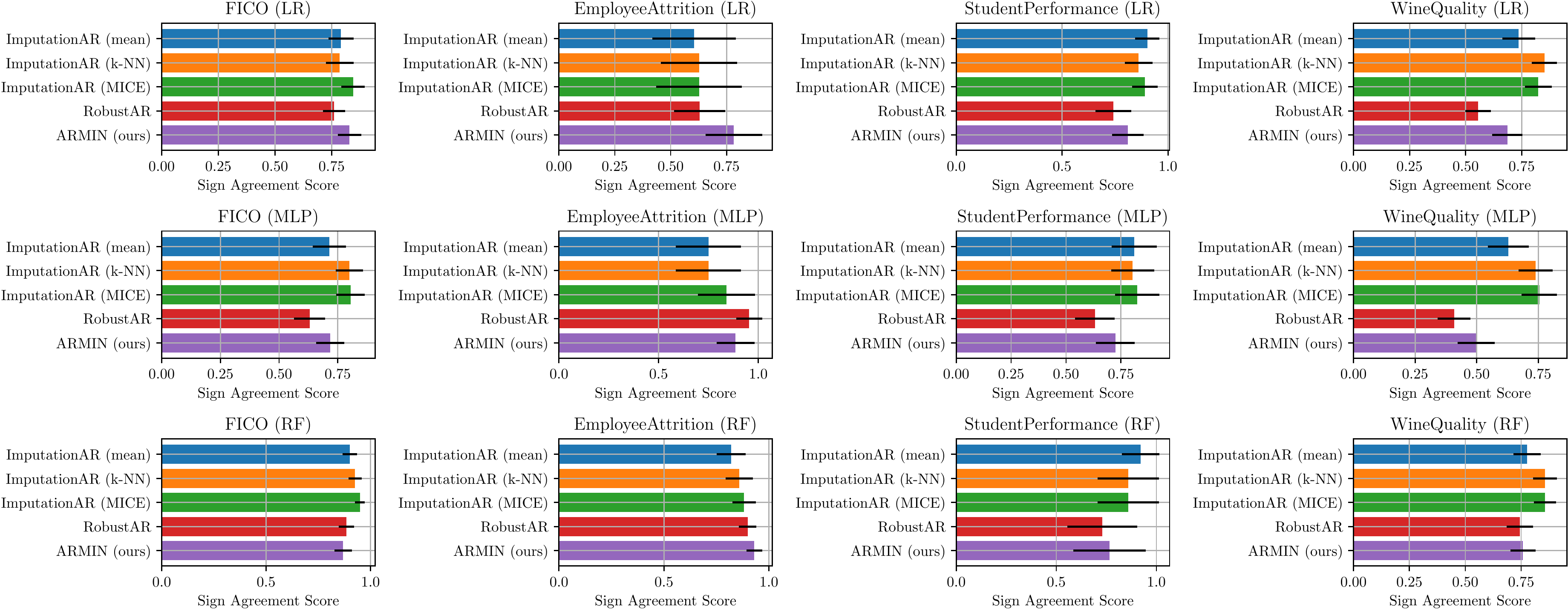}
    \caption{
        Experimental results of the average sign agreement score, where $D_{\ast} = 1$. 
    }
    \label{fig:appendix:exp:sas1}
\end{figure}
\begin{figure}[p]
    \centering
    \includegraphics[width=\textwidth]{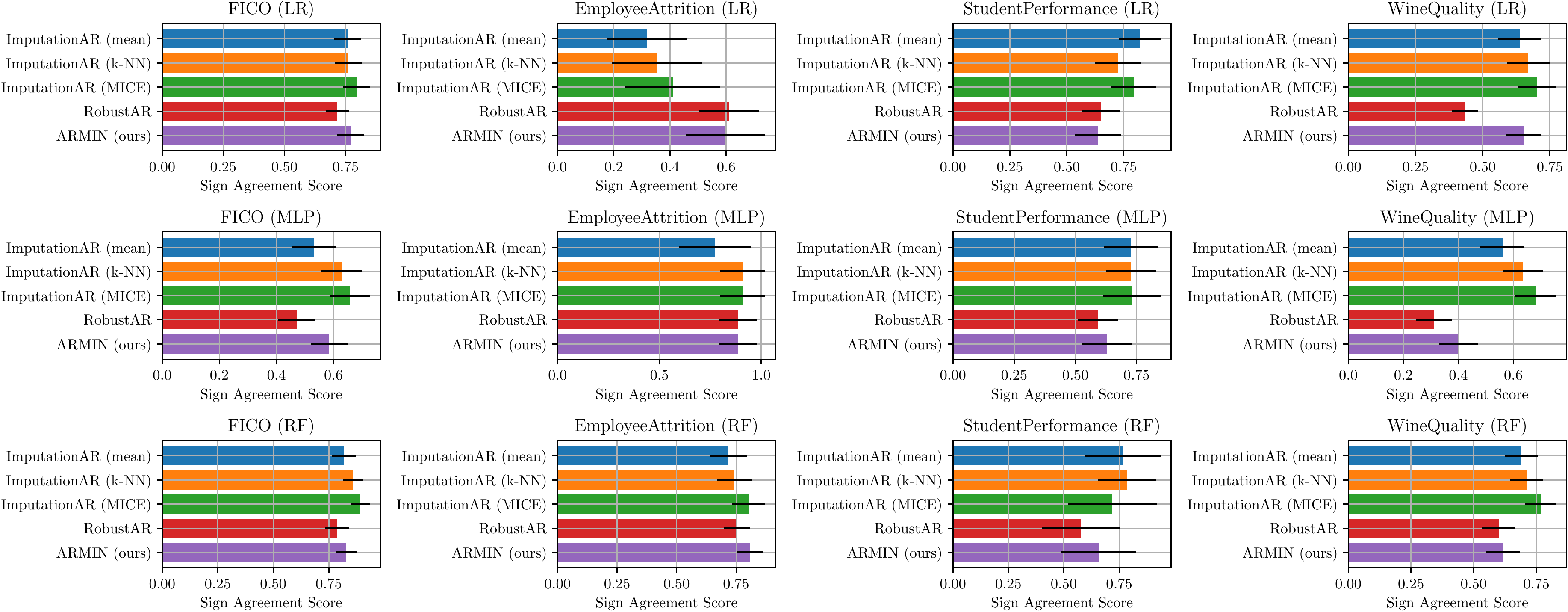}
    \caption{
        Experimental results of the average sign agreement score, where $D_{\ast} = 2$. 
    }
    \label{fig:appendix:exp:sas2}
\end{figure}
\begin{figure}[p]
    \centering
    \includegraphics[width=\textwidth]{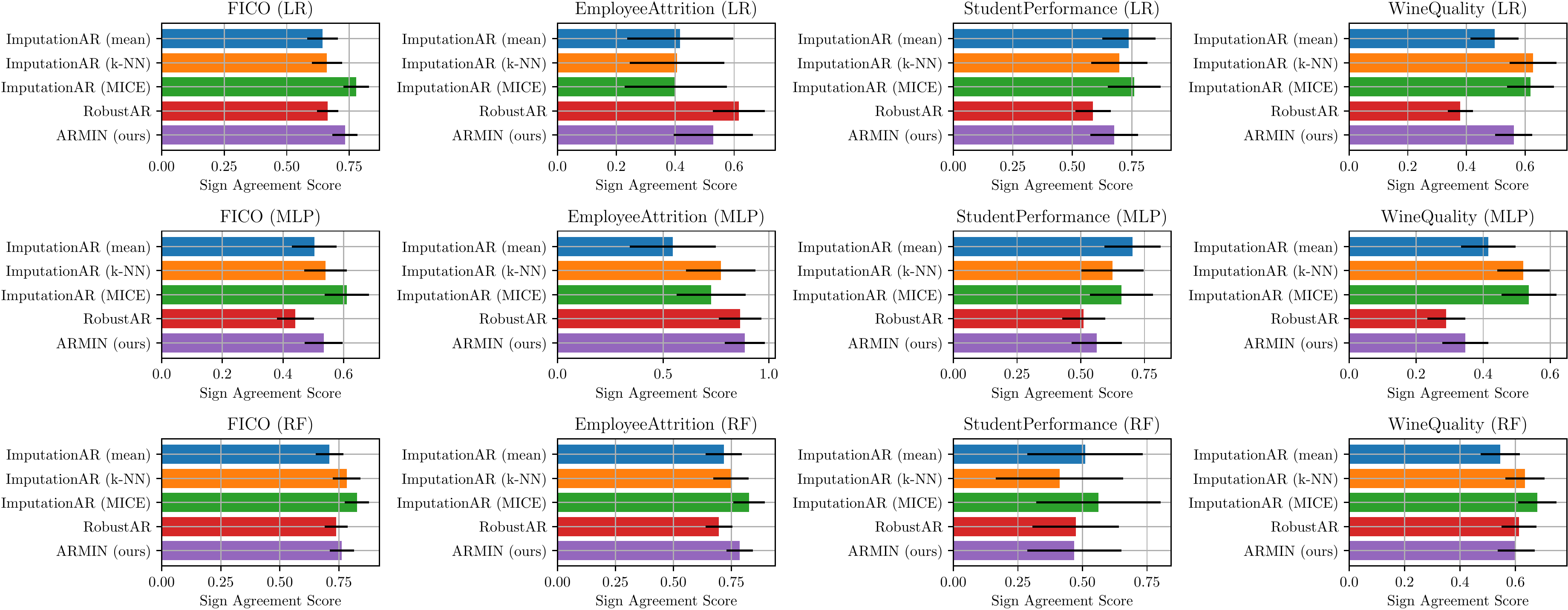}
    \caption{
        Experimental results of the average sign agreement score, where $D_{\ast} = 3$. 
    }
    \label{fig:appendix:exp:sas3}
\end{figure}

\begin{figure}[p]
    \centering
    \includegraphics[width=\textwidth]{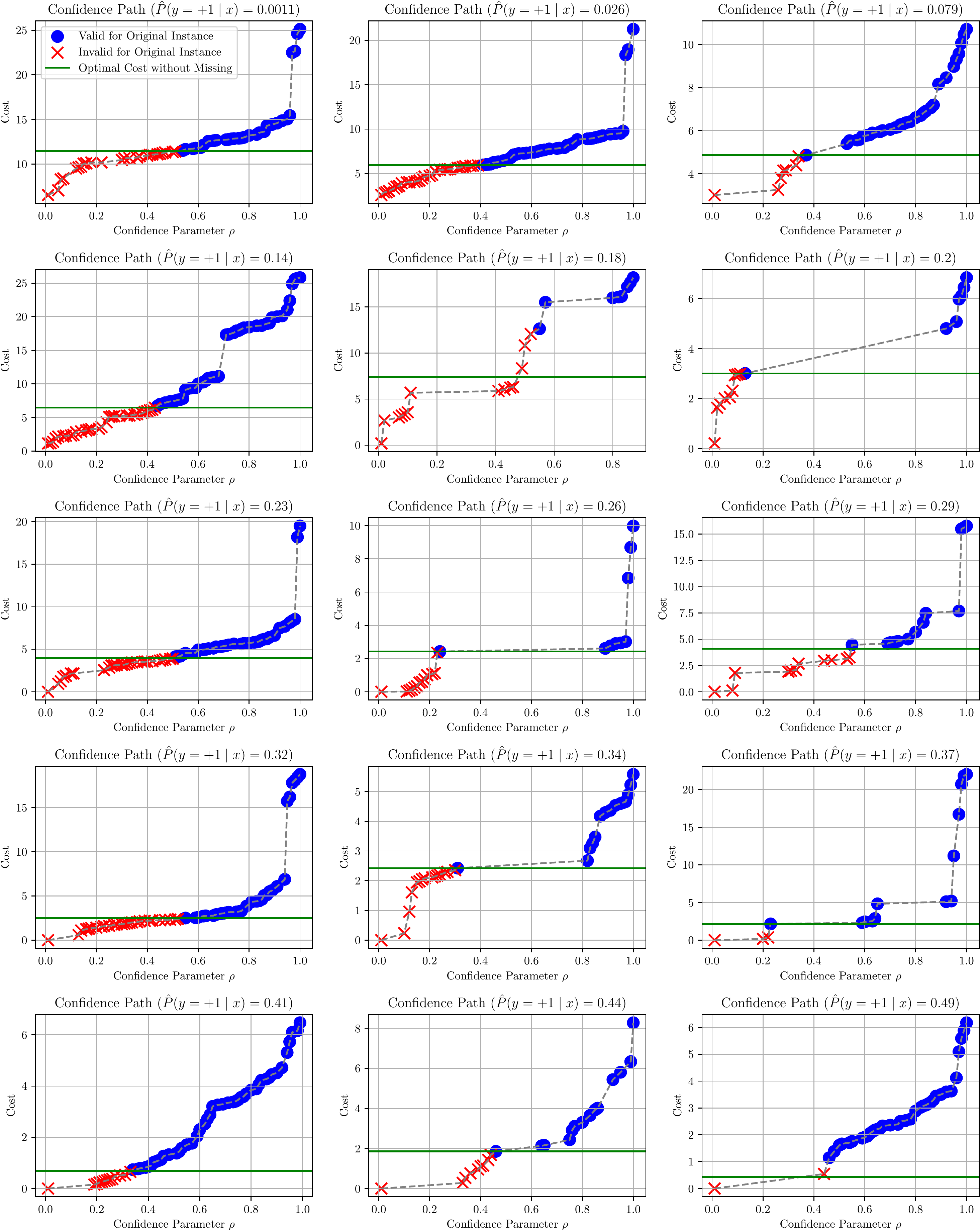}
    \caption{
        Experimental results of our path analyses on the GiveMeCredit dataset under the MAR situation. 
    }
    \label{fig:appendix:exp:pathmar}
\end{figure}
\begin{figure}[p]
    \centering
    \includegraphics[width=\textwidth]{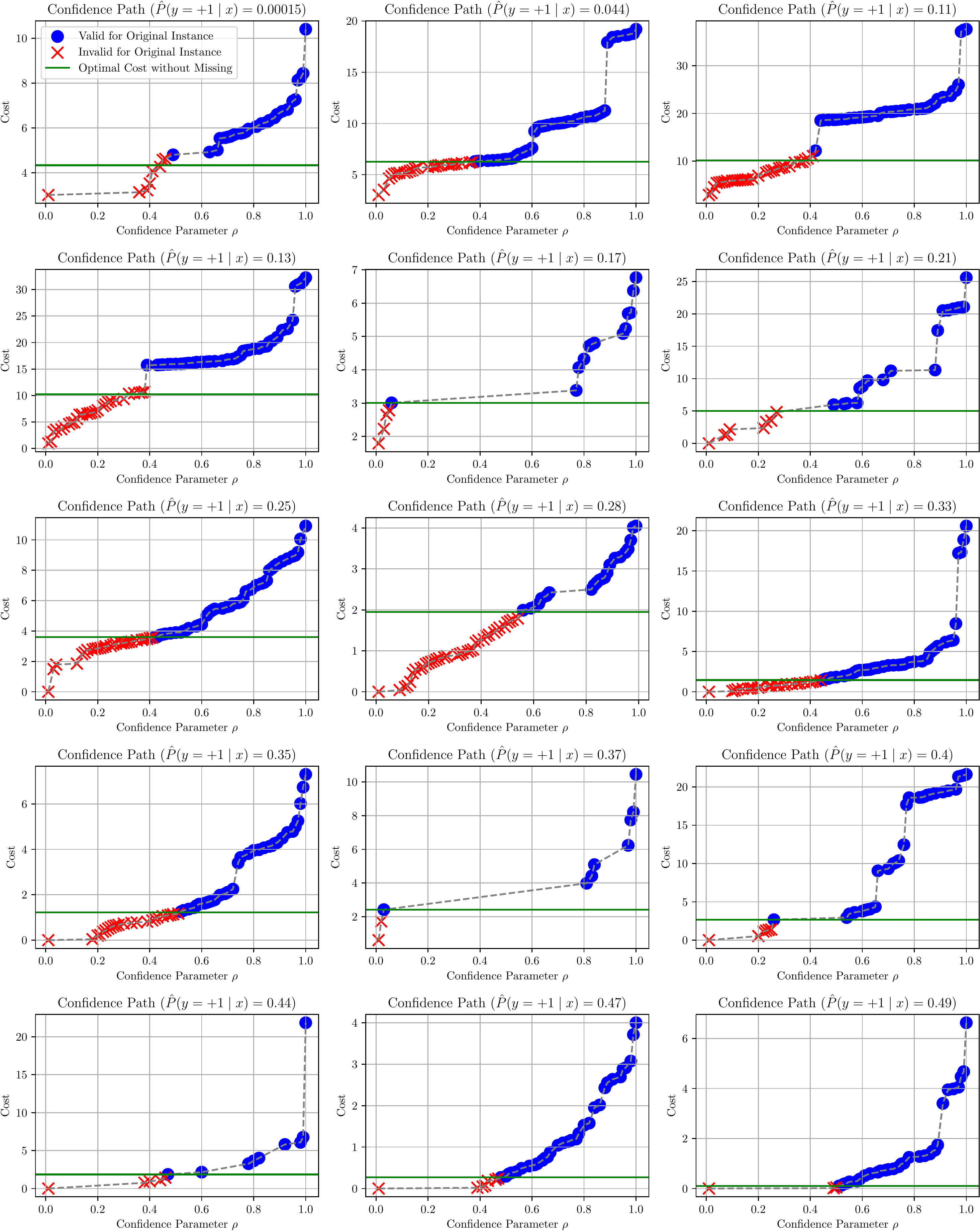}
    \caption{
        Experimental results of our path analyses on the GiveMeCredit dataset under the MNAR situation. 
    }
    \label{fig:appendix:exp:pathmnar}
\end{figure}

\begin{figure}[p]
    \centering
    \includegraphics[width=\textwidth]{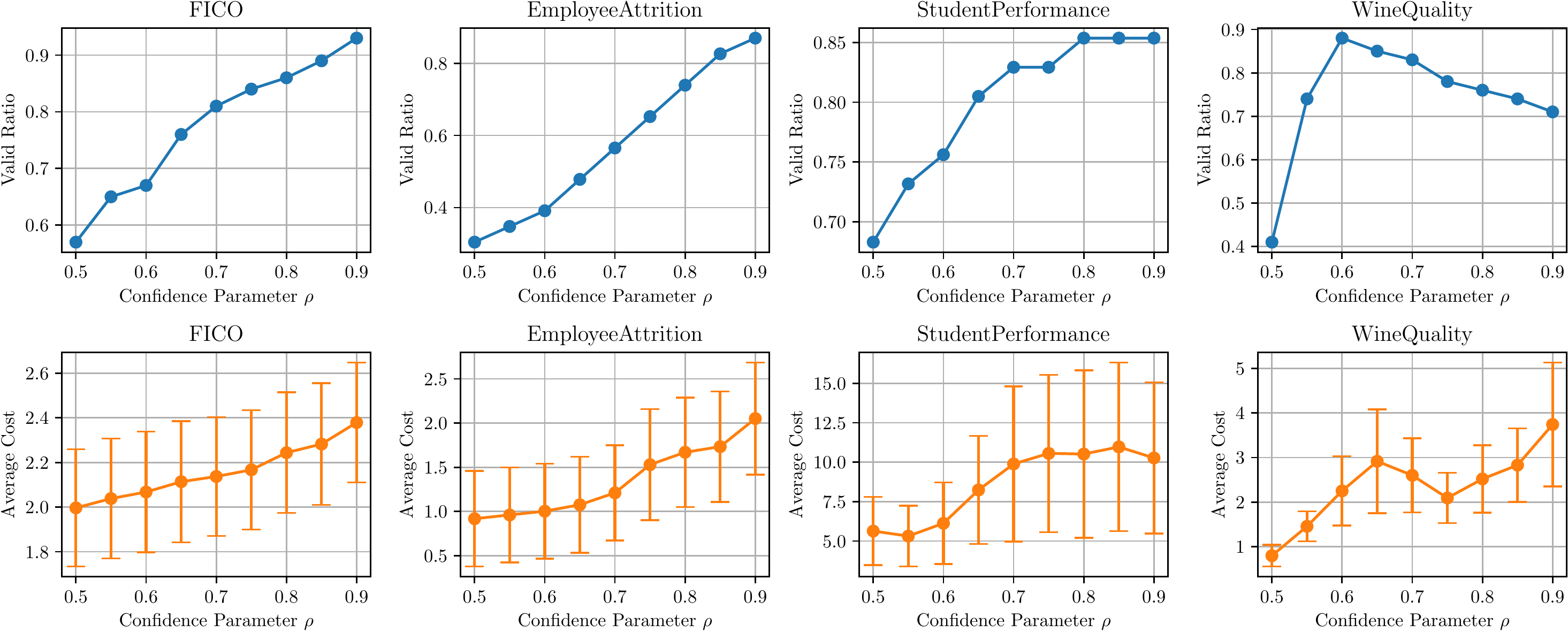}
    \caption{
        Experimental results of our sensitivity analyses of the confidence parameter $\rho$. 
    }
    \label{fig:appendix:exp:confidence}
\end{figure}

\begin{figure}[p]
    \centering
    \includegraphics[width=\textwidth]{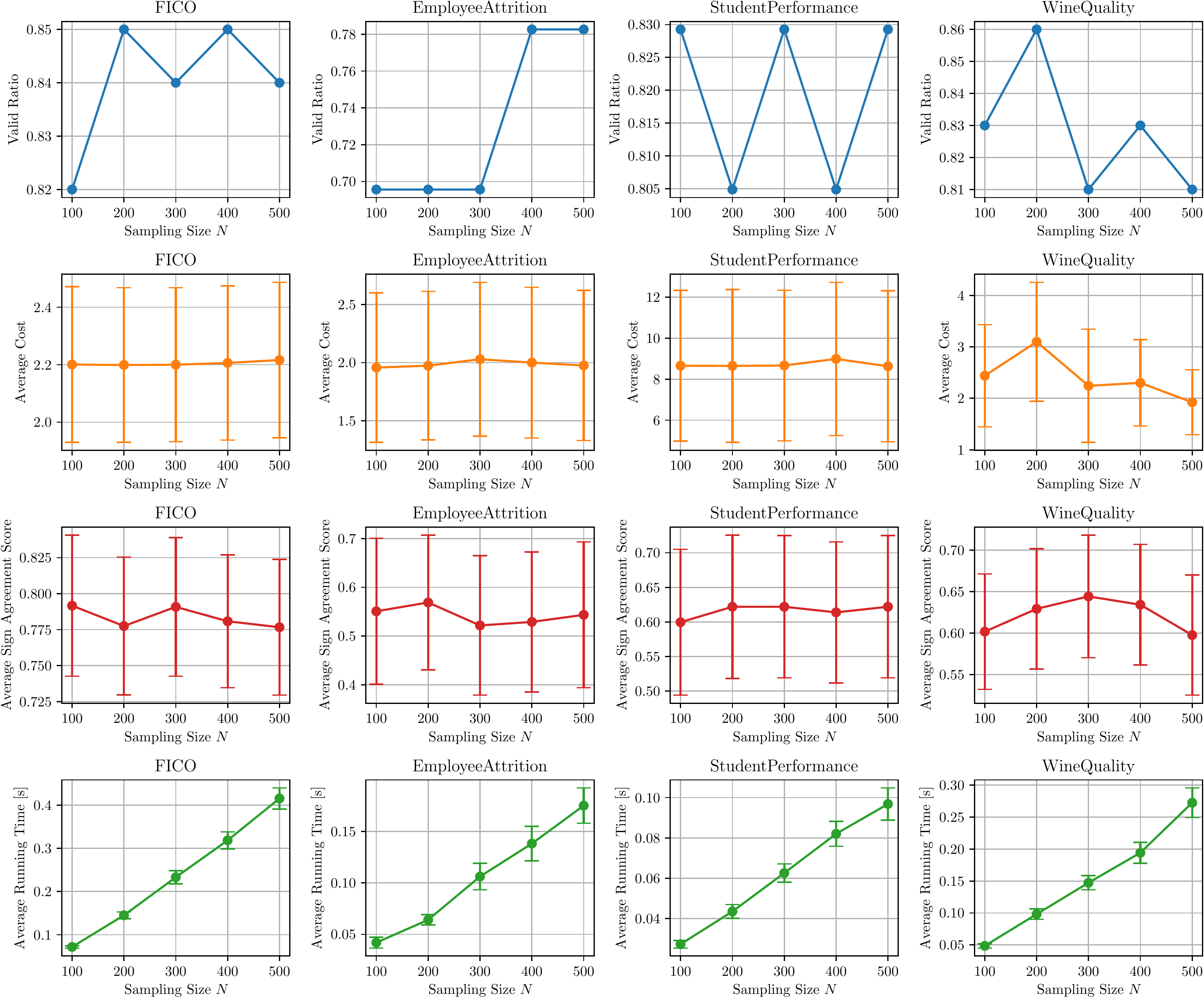}
    \caption{
        Experimental results of our sensitivity analyses of the sampling size $N$. 
    }
    \label{fig:appendix:exp:sampling}
\end{figure}

\end{document}